\pdfoutput=1
\documentclass[11pt]{article}

\usepackage[final]{acl}

\usepackage{times}
\usepackage{latexsym}

\usepackage[T1]{fontenc}
\usepackage[utf8]{inputenc}

\usepackage{microtype}

\usepackage{inconsolata}

\usepackage{graphicx}

\usepackage{algorithm}
\usepackage{algpseudocode}

\usepackage{enumitem}
\usepackage{subcaption}
\usepackage{amsmath,amsfonts,mathtools} 
\usepackage{booktabs}
\usepackage{float}
\usepackage{xcolor}
\definecolor{qzero}{HTML}{7A81FF}  % q0 的颜色：淡蓝色
\definecolor{qone}{HTML}{FA5252}  % q1 的颜色：棕红色
\definecolor{mygray}{HTML}{7F7F7F}

\usepackage{amsmath}
\newtheorem{theorem}{Theorem}

\newtheorem{proof}{Proof}[section]

\title{Finite State Automata Inside Transformers with Chain-of-Thought:\\ A Mechanistic Study on State Tracking}

\renewcommand{\thefootnote}{\fnsymbol{footnote}}
\author{
 \textbf{Yifan Zhang\textsuperscript{1,2}\footnotemark[1]},
 \textbf{Wenyu Du\textsuperscript{3}},
 \textbf{Dongming Jin\textsuperscript{1,2}},
 \textbf{Jie Fu\textsuperscript{4}\footnotemark[2]},
 \textbf{Zhi Jin\textsuperscript{1,2}\footnotemark[2]}
\\
 \textsuperscript{1}Key Laboratory of High Confidence Software Technology (PKU), MOE, China
 \\
 \textsuperscript{2}School of Computer Science, Peking University, China
 \\
 \textsuperscript{3}The University of Hong Kong,
 \textsuperscript{4}Shanghai AI Lab
\\
 \texttt{yifanzhang@stu.pku.edu.cn, fujie@pjlab.org.cn, zhijin@pku.edu.cn}
}

\begin{document}
\maketitle
\footnotetext[1]{Work done during internship at Shanghai AI Lab.}
\footnotetext[2]{Corresponding and co-senior authors.}

\renewcommand{\thefootnote}{\arabic{footnote}}

\begin{abstract}

Chain-of-Thought (CoT) significantly enhances the performance of large language models (LLMs) across a wide range of tasks, and prior research shows that CoT can theoretically increase expressiveness. However, there is limited mechanistic understanding of the algorithms that a Transformer with CoT (denoted as $\tt{Transformer}_{+CoT}$ in this paper) can learn. 
Our key contributions are: (1) We evaluate the state tracking capabilities of $\tt{Transformer}_{+CoT}$ and its variants, confirming the effectiveness of CoT. (2) Next, we identify the circuit (a subset of model components, responsible for tracking the world state), indicating that late-layer MLP neurons play a key role. We propose two metrics, compression and distinction, and show that the neuron sets for each state achieve nearly 100\% accuracy, providing evidence of an implicit finite state automaton (FSA) embedded within the model. 
(3) Additionally, we explore three challenging settings: skipping intermediate steps, introducing data noises, and testing length generalization. Our results demonstrate that $\tt{Transformer}_{+CoT}$ learns robust algorithms (FSAs), highlighting its resilience in challenging scenarios. 
Our code is available at \url{https://github.com/IvanChangPKU/FSA}.

\end{abstract}

\section{Introduction} \label{Introduction}

\begin{figure}[t] 
    \centering
    \includegraphics[width=\columnwidth]{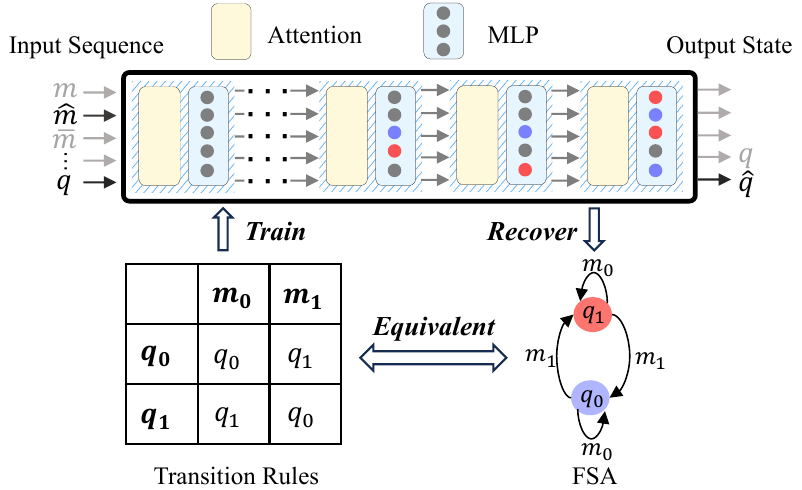}
    \caption{
    % An illustration of one of the simplest state tracking problems, $\mathbb{Z}_2$. After training on data generated by transition rules of $\mathbb{Z}_2$, $\tt{Transformer}_{+CoT}$ successfully recovers an implicit FSA by differentiating two states (\textcolor{qzero}{$q_0$} and \textcolor{qone}{$q_1$}) of the $\mathbb{Z}_2$ using two distinct sets of neurons in late-layer MLPs. 
    An illustration of one of the simplest state-tracking problems, $\mathbb{Z}_2$. After training on sequences generated from the $\mathbb{Z}_2$ transition rules, $\tt{Transformer}_{+CoT}$ successfully recovers an implicit finite state automaton (FSA) by differentiating between two internal states—\textcolor{qzero}{$q_0$} and \textcolor{qone}{$q_1$}—using two distinct and disjoint sets of neurons in the late-layer MLPs. These neuron groups are visually distinguished using colors: neurons corresponding to state \textcolor{qzero}{$q_0$} are marked in \textcolor{qzero}{purple}, those for \textcolor{qone}{$q_1$} in \textcolor{qone}{red}, while neurons not contributing to either state are shown in \textcolor{mygray}{gray}.
    }
    \label{fig:motivation}
\end{figure}

Transformer-based large language models (LLMs \citealp{touvron2023llamaopenefficientfoundation, openai2023gpt4}) revolutionize natural language processing (NLP) by demonstrating significant progress across various tasks. 
However, they still face challenges with basic calculations \citep{zhou2023algorithmstransformerslearnstudy}, complex reasoning \citep{valmeekam2024llmscantplanlrms, han2024folionaturallanguagereasoning}, and regular languages \citep{bhattamishra-etal-2020-ability}. Approaches such as Chain-of-Thought (CoT) prompting \citep{wei2023chainofthoughtpromptingelicitsreasoning} and scratchpads \citep{nye2021workscratchpadsintermediatecomputation} address these limitations by generating intermediate reasoning steps. 
To understand the success of CoT, prior work has analyzed its expressiveness from the perspective of formal language theory and circuit complexity. 
Theoretical works \citep{zhang2024autoregressivechainthought, qiu2024askshallgiventuring, li2024chainthoughtempowerstransformers} demonstrate that incorporating a linear number of intermediate steps increases the expressive power of transformers, enabling them to represent all \textbf{Finite State Automata}, which are a foundational class of automata.

However, theoretical expressiveness indicates only upper and lower bounds on what an architecture can express; it does not guarantee successful learning during training in practice. 
For instance, although recurrent neural networks (RNNs) are theoretically more expressive than transformers in Chomsky's computational hierarchy—being capable of recognizing all regular languages, whereas transformers fail to recognize certain types (\textit{e.g.}, periodic finite-state languages) \citep{delétang2023neuralnetworkschomskyhierarchy}—they often fail to outperform transformers in practice.
This discrepancy is primarily due to issues such as vanishing gradients, which hinder learning long-term dependencies, and difficulties in parallelization, which limit computational efficiency \citep{6795963, 10.5555/3295222.3295349, 10.5555/3495724.3495883}. 
Consequently, some studies investigate the expressiveness of these architectures through the lens of learnability by measuring performance in language modeling. For example, \citet{liu2023transformerslearnshortcutsautomata} demonstrate that training transformers with recency-biased scratchpads improves sequential accuracy. However, even near-perfect next-token prediction accuracy does not imply that generative models reconstruct a true world model \citep{vafa2024evaluatingworldmodelimplicit}. This raises a critical question: \textbf{Does CoT help transformers recover a world model in the form of FSAs, or do they merely learn shortcuts?}

To address this question, we extend the study of learnability beyond accuracy improvements, performing an internal mechanistic analysis of CoT's success. Specifically, we focus on \textit{state tracking}, a foundational task to evaluate expressiveness \citep{merrill2024illusionstatestatespacemodels}. In state tracking, a sequence of updates modifies the world state, which is represented as an FSA. The goal is to determine the final state after applying all updates sequentially. State tracking is a core capability of generative models and supports many downstream tasks — such as entity tracking \citep{kim-schuster-2023-entity}, chess \citep{toshniwal2022chesstestbedlanguagemodel}, and map navigation \citep{liu2023agentbenchevaluatingllmsagents}. 
Figure \ref{fig:motivation} illustrates \(\mathbb{Z}_2\), one of the simplest state tracking problems\footnote{The state tracking problem \(\mathbb{Z}_2\) is equivalent to parity, a formal language describing binary sequences with specific evenness or oddness properties.}, along with its corresponding FSA and transition rules. 

We begin by comprehensively evaluating the state tracking capabilities of $\tt{Transformer}_{+CoT}$, comparing it with other models (RNNs), transformer variants (\textit{e.g.}, those with recurrence \citep{fan2021addressinglimitationstransformersfeedback, 9878733}), and CoT variants (\textit{e.g.}, implicit CoT \citep{goyal2024thinkspeaktraininglanguage}). Empirically, we show that $\tt{Transformer}_{+CoT}$ is the only model capable of efficiently learning state tracking for sequences of arbitrary lengths across three groups: $\mathbb{Z}_{60}$, $A_4 \times \mathbb{Z}_5$, and $A_5$, in both in-distribution and out-of-distribution settings. 

Next, to provide a mechanistic explanation for this success, we apply interpretability techniques to analyze the algorithms learned by $\tt{Transformer}_{+CoT}$. Using activation patching \citep{10.5555/3495724.3496763}, we identify the circuits (specific model components) responsible for state tracking and observe that $\tt{Transformer}_{+CoT}$ relies heavily on late-layer MLP neurons. These neurons can be effectively grouped into states based on transition rules. To quantify this, we introduce two metrics: compression and distinction. Compression measures the similarity of representations for the same state under different input prompts, while distinction quantifies the separation between different states, even when their inputs are similar. We find nearly 100\% accuracy on both metrics at every intermediate step, providing strong evidence that the model reconstructs the world model (\textit{i.e.}, FSAs). 
For instance, in Figure \ref{fig:motivation}, $\tt{Transformer}_{+CoT}$ compresses inputs corresponding to two states (\textit{i.e.} \textcolor{qzero}{$q_0$} and \textcolor{qone}{$q_1$}) by activating two distinct sets of neurons. 
% {Transformer+CoT effectively encodes inputs that correspond to the same state by selectively activating a distinct subset of neurons. Simultaneously, it maintains clear separability by ensuring that this subset of neurons is uniquely activated for different states, thereby preserving state-specific representations.}
% Additionally, we observe shifts in attention patterns as the number of intermediate steps increases. 

In real-world tasks that involve state tracking, state transitions are often implicit, and reasoning structures tend to be imperfect.
For example, mathematical reasoning is related to state tracking, which involves keeping track of the problem goal, and chain-of-thought annotations in existing datasets—such as OpenWebMath—contain skipping or noise.
To evaluate robustness in challenging scenarios, we assume three experimental settings: skip-step reasoning, noisy scratchpads, and length generalization. 
Through controlled experiments, we show that $\tt{Transformer}_{+CoT}$ learns robust algorithms even in noisy environments, suggesting that the underlying FSAs exhibit strong resilience. 
% Such robustness enhances the model's capability to tackle real-world tasks involving state tracking. 
% For example, mathematical reasoning is related to state tracking, which involves keeping track of the problem goal, yet chain-of-thought annotations in existing datasets—such as OpenWebMath\footnote{\url{https://github.com/keirp/OpenWebMath}}—contain skipping or noise.

In summary, this work is the first to extend the study of learnability and expressiveness through mechanistic interpretation on state tracking, uncovering the underlying algorithms of $\tt{Transformer}_{+CoT}$. Our contributions are as follows: 

\begin{enumerate}[noitemsep]
    \item We conduct a comprehensive evaluation of the state tracking capabilities of $\tt{Transformer}_{+CoT}$, demonstrating its unique ability to track states of arbitrary lengths across multiple groups ($\mathbb{Z}_{60}$, $A_4 \times \mathbb{Z}_5$, and $A_5$) in both in-distribution and out-of-distribution settings.
    \item Using interpretability techniques, including activation patching, we analyze the learned algorithms in $\tt{Transformer}_{+CoT}$. We identify the activation of late-layer MLP neurons and classify them into states based on transition rules, achieving nearly 100\% accuracy in metrics of compression and distinction, which confirms the model's reconstruction of the world model (FSAs).
    \item We explore $\tt{Transformer}_{+CoT}$ in three challenging settings and find that it learns resilient algorithms capable of effective state tracking in noisy conditions. 
\end{enumerate}

\section{Background} \label{Background&Related Work}

\subsection{FSA and State Tracking}

We adopt the conventional definition of a finite state automaton (FSA) as a tuple $\mathcal{A} = (\Sigma, Q, q_0, \delta)$, where $\Sigma$ is the alphabet, $Q$ is a set of states, $q_0$ is the initial state, and $\delta$ is the transition function \citep{10.5555/1454320}. 
Formally, state tracking can be framed as solving a word problem on a finite monoid $(M, \cdot)$, where the objective is to compute the product $m_1 \cdot m_2 \cdot \ldots \cdot m_n \in M$ \citep{merrill2024illusionstatestatespacemodels}. 
When $M$ is finite, the computation can be carried out by a corresponding finite state automaton $(M, M, e, \delta)$, where the identity element $e$ acts as the initial state and the transition function is defined as $\delta(m_1, m_2) = m_1 \cdot m_2$ for all $m_1, m_2 \in M$.

As generative models, transformers augmented with chain-of-thought generate state sequences $(q_1 \ldots q_n) \in M^*$ conditioned on input sequences $(m_1 \ldots m_n) \in M^*$. 
Our work centers on word problems in the context of groups, which are monoids with inverses. 
We specifically investigate two structures: the cyclic group $\mathbb{Z}_k$, defined by addition modulo $k$, and the alternating group $A_k$, a subgroup of the symmetric group $S_k$ comprising all even permutations of $k$ elements.

\subsection{Mechanistic Interpretation of MLPs} 
% The use of mechanistic interpretability \citep{rai2024practicalreviewmechanisticinterpretability, ferrando2024primerinnerworkingstransformerbased} to explain LLMs is an emerging research direction. 
% This approach employs various techniques, including logit lens \citep{geva2021transformerfeedforwardlayerskeyvalue}, probing \citep{gurnee2023findingneuronshaystackcase}, causal mediation analysis \citep{wang2022interpretabilitywildcircuitindirect}, sparse autoencoders \citep{cunningham2023sparseautoencodershighlyinterpretable}, and visualization \citep{cooney2023circuitsvis}, to identify and analyze the features and circuits of LLMs. 
\citet{geva-etal-2022-transformer} demonstrate that multilayer perceptrons (MLPs) contribute additive updates to the residual stream, which can be decomposed into weighted sums of sub-updates.
In particular, given input $\mathbf{x^l}$ at layer $l$, the MLP can be expressed using parameter matrices $\mathbf{K^l}, \mathbf{V^l} \in \mathbb{R}^{d_{\text{mlp}} \times d_{\text{m}}}
$, where $d_{\text{mlp}}$ is the MLP intermediate dimension and $d_{\text{m}}$ is the model dimension.
Additionally, a non-linear activation function $f$ is applied as follows:
\[\text{MLP}^l(\mathbf{x^l}) = f(\mathbf{K^l}\mathbf{x^l})\mathbf{V^l} \]
Expanding this further, it can be decomposed as: 
\[\text{MLP}^l(\mathbf{x^l}) = \sum_{j=1}^{d_{\text{mlp}}} f(\mathbf{x^l} \cdot \mathbf{k_j^l}) \mathbf{v_j^l} = \sum_{j=1}^{d_{\text{mlp}}} m_{j}^l \mathbf{v_j^l}
\]
where $\mathbf{k^l_j} \in \mathbb{R}^{d_{\text{m}}}$ and $\mathbf{v^l_j} \in \mathbb{R}^{d_{\text{m}}}$ correspond to the $j\text{-th}$ row vectors of $\mathbf{K^l}$ and $\mathbf{V^l}$, respectively. 
The scalar $m_{j}^l = f(\mathbf{x^l} \cdot \mathbf{k_j^l})$ represents the activation coefficient for the neuron $\mathbf{v_j^l}$. 
Notably, when these sub-updates $m_j^l \mathbf{v_j^l}$ are projected into the vocabulary space using the logit lens \citep{geva2021transformerfeedforwardlayerskeyvalue}, they can be interpreted in a human-understandable way.
% \citet{nikankin2024arithmeticalgorithmslanguagemodels} provide a high-resolution understanding of the mechanism that LLMs use to answer arithmetic prompts using a set of techniques including activation patching, logit lens, etc. 
% And \citet{saparov2024transformersstrugglelearnsearch} designs a novel mechanistic interpretability technique to reconstruct the computational graph, through utilizing activation patching. 
% In this work, we aim to provide a mechanistic interpretation on the algorithm transformers with chain-of-thought have learned, to determine whether and to what extent the model is utilizing a correct algorithm to keep track of world state, rather than based on heuristics. 

\section{Evaluating State Tracking Capability Across Architectures} \label{sec:model performance}

\begin{figure*}[ht] 
    \centering
    \includegraphics[width=\linewidth]{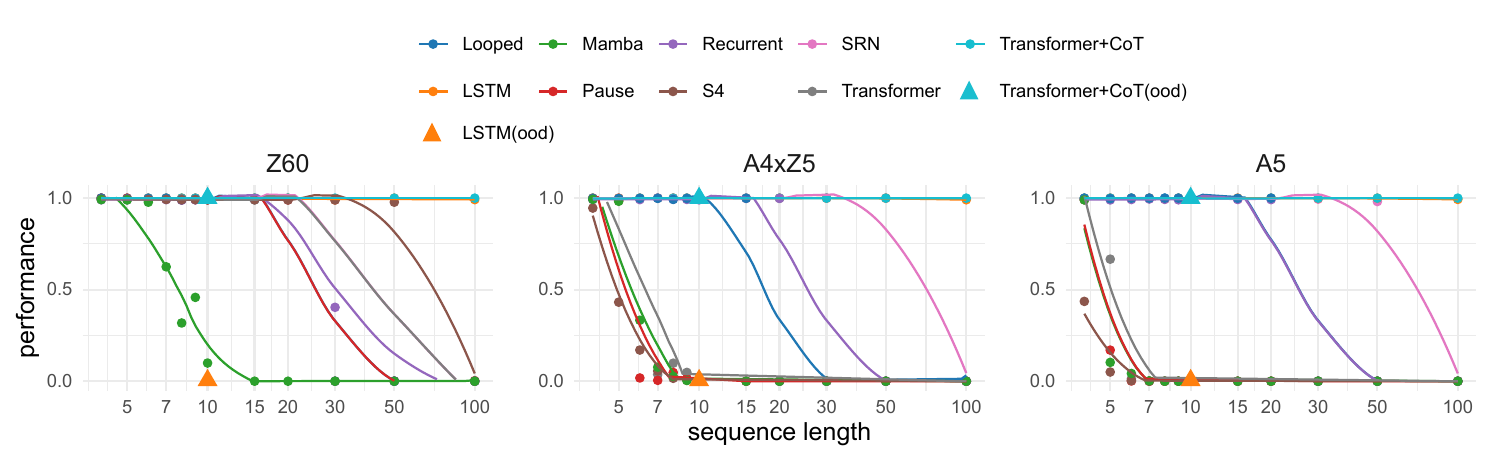}
    \caption{Model accuracy across sequence lengths for $\mathbb{Z}_{60}, A_4 \times \mathbb{Z}_5 $ and $A_5$.
    The x-axis shows sequence lengths (2 to 100), and the y-axis shows sequence accuracy. 
    Each line represents a different model (see legend). Dots indicate in-distribution performance, while triangles indicate out-of-distribution performance (group elements sampled from the full set, including mixed sequences not encountered during training, though within the same length range).}
    \label{fig:model_accuracy}
\end{figure*}
% Previous work \citep{liu2023transformerslearnshortcutsautomata, merrill2024illusionstatestatespacemodels, grazzi2024unlockingstatetrackinglinearrnns} mainly model the state tracking problem as token-tagging task, where the sequence-to-sequence neural network (encoder) maps the input sequences $m_1 \ldots m_n \in M^*$ to corresponding state sequences $q_1 \ldots q_n \in M^*$. 
% However, the dominant large language models accept transformers as generative model (decoder), and whether transformers with chain-of-thought can accomplish this task remains a subject of further investigation. 

Besides $\tt{Transformer}_{+CoT}$, there are various theoretical works \citep{zhang2024autoregressivechainthought, fan2024loopedtransformerslengthgeneralization, 9878733, fan2021addressinglimitationstransformersfeedback} attempting to inject recurrence into transformers, while another line of work \citep{goyal2024thinkspeaktraininglanguage, hao2024traininglargelanguagemodels} proposing modifications of chain-of-thought (implicit chain-of-thought in contrast to explicit chain-of-thought). 
In this section, we will explore the state tracking capability with an empirical lens: \textit{can transformers with/without CoT, and their variants, learn state tracking?}

\paragraph{Datasets: $A_5,A_4 \times \mathbb{Z}_5, \mathbb{Z}_{60}.$}

Following the formal definition of state tracking in \citep{merrill2024illusionstatestatespacemodels, grazzi2024unlockingstatetrackinglinearrnns}, we model state tracking as word problems, and consider three kinds of groups with increasing difficulty: $\mathbb{Z}_{60}$, an abelian group encoding mod-60 addition, $A_4 \times \mathbb{Z}_5$, a non-abelian but solvable group\footnote{Formally, a finite group $G$ is solvable if it has a subnormal series $1=G_0 < G_1 < \ldots < G_k=G$ such that each factor $G_i/G_{i-1}$ is abelian.}, which is a direct product group of one alternating group $A_4$ (a subgroup of the symmetric group $S_4$ containing only even permutations) and one cyclic group $\mathbb{Z}_5$, and $A_5$, the alternating group on five elements, which is the smallest non-solvable subgroup. 
With the same number of elements 60, the three groups belong to $\mathsf{TC}^0$, $\mathsf{TC}^0$, and $\mathsf{NC}^1$-complete respectively, with varying difficulty mainly deriving from the complexity of learning the group multiplication operation.

\paragraph{Architectures.}

We denote $\tt{Transformer}$ as a GPT2 architecture \citep{radford2019language} with a bounded number of layers. When augmented with chain-of-thought, it is denoted as $\tt{Transformer}_{+CoT}$. 
To compare with other models, we also consider recurrent neural networks (RNNs \citealp{10.5555/553011} and LSTMs \citealp{6795963}), S4 \citep{gu2022efficientlymodelinglongsequences}, Mamba \citep{gu2024mambalineartimesequencemodeling}, implicit chain-of-thought: transformers with pause (denoted as Pause) \citep{goyal2024thinkspeaktraininglanguage} and other variants of transformers: standard recurrent transformer (denoted as Recurrent) \citep{9878733}, looped transformer (denoted as Looped) \citep{fan2024loopedtransformerslengthgeneralization}. 

\paragraph{Task Formulation.}

In order to disentangle recurrence introduced by CoT from those arising from architectural modifications (\textit{e.g.}, the standard recurrent transformer, where recurrence is inherent to the model architecture), we employ two different task formulations for evaluation: the Token-Tagging (TT) task and the Language Modeling (LM) task. 
In the TT task, each input token $m_i$ is annotated with its corresponding world state $q_i$. 
In contrast, the LM task requires the model to autoregressively generate the sequence of state transitions $(q_1 \ldots q_n)$.
Specifically, only $\tt{Transformer}_{+CoT}$ and implicit CoT (Pause) are evaluated using the LM task, whereas other models and variants the TT task.
The reason of the difference is that, we use the same set of labels whatever the task is, so as to eliminate the label-related influences on supervised training—similar to the ``hint'' setting in \citep{li2024chainthoughtempowerstransformers}.

\paragraph{Experimental Setup.}

We utilize the Python package \texttt{Abstract Algebra}\footnote{\url{https://github.com/alreich/abstract_algebra}} to randomly generate three distinct types of group-based datasets. 
Each dataset consists of input sequences $(m_1 \ldots m_n)$ and corresponding state sequences $(q_1 \ldots q_n)$, where state transitions follow the multiplication operations of three different algebraic groups: $A_5, A_4 \times \mathbb{Z}_5, \mathbb{Z}_{60}$.
For comparison, all models are configured with a single layer and a model dimension of 512. 
Each model is trained on 1,000,000 sampled sequences of length $n$, for successively larger values of $n$, and evaluated the sequence accuracy on a held-out validation set.  
We set a maximum of 500 training epochs and employ early stopping once the validation accuracy reaches 99\%.

\paragraph{In-Distribution Performance.} \label{Performance in Distribution}

We hold a comprehensive evaluation of models' state tracking capability on word problems, with the same model depth and dimension. 
% We use sequence accuracy as a metric, where the generated sequence is true only when same to ground truth sequence. 
Figure~\ref{fig:model_accuracy} gives different models' performance across different input sequence length and different groups. 
We draw several conclusions:

\begin{enumerate}[noitemsep]
    \item Consistent with theoretical study \citep{liu2023transformerslearnshortcutsautomata, merrill2024illusionstatestatespacemodels}, $\tt{Transformer}$ and state-space models, such as S4 and Mamba, are incapable of expressing arbitrary length $A_5$ word problems, in contrast to RNN and LSTM.
    \item CoT definitely extends the expressive power of $\tt{Transformer}$. 
    In contrast to Pause, which fails with longer sequences, $\tt{Transformer}_{+CoT}$ can efficiently learn word problems of arbitrary length across multiple groups.
    % \item which allows for parallel training and a smoother flow of gradients like Transformer due to no architecture change.
    % \item Expressiveness doesn't equal learnability. 
    % Although $A_4 \times \mathbb{Z}_5$ belongs to $\mathsf{TC}^0$, it can not be learned by models with circuit complexity $\mathsf{TC}^0$ on sequences of arbitrary length. 
    % Moreover, although some models (RNN, Recurrent and Looped) are theoretically capable of expression, they exhibit instability or even failure during training on longer word problems. 
    \item $\tt{Transformer}_{+CoT}$ achieves a dual win in both expressiveness and learnability. 
    Unlike architectural modifications such as Looped or Recurrent, $\tt{Transformer}_{+CoT}$ maintains the original $\tt{Transformer}$ architecture without modification, and thus preserves the inherent advantages of the standard $\tt{Transformer}$, including the parallel training and smoother gradient flow.
    While alternative architectures may postpone accuracy degradation on longer input sequences compared to the standard $\tt{Transformer}$, they ultimately fail to converge on sequences of arbitrary length.
\end{enumerate}

\paragraph{Out-of-Distribution Performance.} \label{subsec: ood}

\citet{liu2023transformerslearnshortcutsautomata} analyze transformers' failure in out-of-distribution of parity, and argue that transformers learn a shortcut solution, which compute the parity by counting the number of 1s in the input sequences and compute mod-$2$, thus failing to generalize to sequences with unseen sum of 1s. 
\citet{zhang2024autoregressivechainthought} discuss the role of chain-of-thought in computability, and point that CoT simulates the recurrent connection by iteratively encode and decode back and forth between states and tokens. 
% \emph{The key of Transformer+CoT's expressiveness on state tracking is that previous computation can be retrieved, through concatenating previous step state to the end of scratchpad}.
The expressiveness of $\tt{Transformer}_{+CoT}$ in state tracking stems from its ability to retrieve prior computations by appending the previous step's state to the end of the scratchpad.

To investigate whether $\tt{Transformer}_{+CoT}$ learns an algorithm based solely on the input sequences $(m_1 \ldots m_n)$ or combines input and scratchpad as theoretical work expects, we divide the elements of the three groups into proper subsets.
% The designed distribution shift experiments train the model on sequences where each input sequence $m_1 \ldots m_n$ belongs to a specific proper subset, and different sequences may belong to different subsets.
% The evaluation is conducted on the sequences sampled from the full set.
% The designed distribution shift experiments train the model on sequences, where each sequence $m_1 \ldots m_n$ belongs to any proper subset, with different sequences potentially belonging to different subsets. 
% Evaluation is performed on sequences $m_1 \ldots m_n$ sampled from the full set.
% We set up the following training set: $m$ belongs to one proper subset for one sequence, but different sequences may belong to different subsets, and evaluation is performed on the full set.
And we train the model on sequences with $m$ belonging to one proper subset, but evaluate on the full set. 
We stress that, through restricting the input sequences into separate subsets, the possible state sequences remain the same, for the reason that any subset can express the whole group through group operation.
If the model learns a shortcut solution, that attends to only input, it can not generalize to sequences with group elements sampled from the full set, because the model has not seen mixed input sequences in the training set. 
As previously outlined, both LSTM and $\tt{Transformer}_{+CoT}$ successfully learn all word problems in-distribution, and we test their performance out-of-distribution. 
Results in Figure~\ref{fig:model_accuracy} show that $\tt{Transformer}_{+CoT}$ achieves perfect generalization on three groups $\mathbb{Z}_{60}, A_4 \times \mathbb{Z}_5, A_5$ in contrast to LSTM, implying that the model attends to not only input but also scratchpad.
We provide more experimental settings in Appendix~\ref{groups}. 
The out-of-distribution performance eliminates the possibility that the model learns specific shortcuts similar to those \citet{liu2023transformerslearnshortcutsautomata} find on $\mathbb{Z}_2$ group, but \textit{this does not rule out the possibility that other shortcuts exist, which necessitates an interpretation on the mechanism.}

\section{Mechanism Inside Transformers with Chain-of-Thought: FSAs} \label{sec:mechanism}

Both transformers with chain-of-thought and recurrent neural networks achieve perfect performance in-distribution, and the former even generalizes well out-of-distribution. 
Due to the transformer's black-box nature, the mechanism behind its state tracking implementation remains unknown. Consequently, we cannot answer questions such as what algorithm the model has learned to keep track of the world state, and how well the model can generalize. 
Considering that the word problems involve a series of state transitions, the model's state computation is dynamic and sequential while generating intermediate steps. 
In this section, we first try to analyze the circuit used by transformers with chain-of-thought to keep track of the world state, and then conduct a deeper component analysis to interpret the mechanism.

\subsection{Circuit Localization} \label{subsec: circuit}

To understand the mechanism of state tracking, we will first localize the circuit (a subset of components) using activation patching \citep{10.5555/3495724.3496763}. 
% We formalize the word problems as prompt $\mathbf{p_i} = m_1 \ldots m_n | q_i \ldots q_{i-1}$ and resulting state token $q_i$ at $i\text{-th}$ step.\footnote{In this paper, the scratchpad step index starts at 0.}
We formalize each word problem as a prompt $\mathbf{p_i} = (m_1 \ldots m_n | q_1 \ldots q_{i-1})$, where $(m_1 \ldots m_n)$ is the problem description and $q_i$ is the resulting state token at $i\text{-th}$ step.\footnote{In this paper, the scratchpad step index starts at 0.}
% At each intermediate step, we sample a prompt $\mathbf{p_i}$ with its corresponding resulting state $q_i$, and then sample a counterfactual prompt $\mathbf{p_i'}$, resulting in a different resulting state $q_i'$. 
At each intermediate step, we sample a prompt $\mathbf{p_i}$ with its corresponding state token $q_i$, and then sample a counterfactual prompt $\mathbf{p_i'}$ that yields a different state token $q_i'$.
% We hold intervention experiments, by replacing the activation of a specific MLP layer or attention head with the pre-computed for $\mathbf{p_i'}$, and then assessing how this impacts the probabilities of answer tokens.
We perform intervention experiments by replacing the activation of a specific MLP layer or attention head with the pre-computed activation from $\mathbf{p_i'}$, and then assess how this affects the probabilities of answer tokens.
% To access the importance of one component at $i\text{-th}$ step, we average the following intervention effect (IE) metric \citep{nikankin2024arithmeticalgorithmslanguagemodels} across all prompts, which is the mean of impacts of $q_i$ and $q_i'$: 
To assess the importance of a component at $i\text{-th}$ step, we compute the following intervention effect (IE) metric \citep{nikankin2024arithmeticalgorithmslanguagemodels}, averaged over all prompts. 
The IE measures the mean relative change in the model’s output probabilities for both $q_i$ and $q_i'$:
\[
    \text{IE}(q_i, q_i') = 
    \frac{1}{2} \left[ 
        \frac{\text{P}^*(q_i') - \text{P}(q_i')}{\text{P}(q_i')} 
        + \frac{\text{P}(q_i) - \text{P}^*(q_i)}{\text{P}^*(q_i)} 
    \right]
\]
% where $\text{P}$ and $\text{P}^*$ are the pre- and post-intervention probability distributions.
where $\text{P}(q)$ and $\text{P}^*(q)$ denote the model's predicted probability for token $q$ before and after the intervention, respectively.
% Having identified the component of the circuit, we have to know in which position the model promotes the correct answer token. 
% To understand in which position the model promote the correct answer, we hold linear probing in each layer at each position. 
% We train a linear classifier at each position and layer, which takes corresponding model activation as input and generates a probability distribution over all group elements. 
% The classifier's performance is then assessed on a separate test set, measuring to what extent the correct answer can be derived from the output representation at the layer and position. 

% We localize the circuit $\tt{Transformer}_{+CoT}$ to keep track of world state at each intermediate step in $A_5$ word problems.
% As Figure~\ref{fig:ie&probe} shows, we find that the circuit across each intermediate step mainly consists of MLPs rather than attention heads.
We localize the circuit in $\tt{Transformer}_{+CoT}$ responsible for tracking the world state at each intermediate step in $A_5$ word problems.
As shown in Figure~\ref{fig:ie&probe}, the circuit responsible for state tracking primarily consists of MLPs rather than attention heads.
% And across each intermediate step, the circuit has hardly changed, that is, the first MLP (MLP0 \citealp{mcdougall2023copysuppressioncomprehensivelyunderstanding}) at the position of $m_i$ and late-layer MLPs at the last position play a significant role in state tracking at the $i\text{-th}$ step.
Moreover, the circuit remains largely consistent across steps: MLP0 \citep{mcdougall2023copysuppressioncomprehensivelyunderstanding} at the position of $m_i$ and late-layer MLPs at the last position consistently play a key role in state tracking at the $i\text{-th}$ step.
% Moreover, at the last position, the late-layer MLPs promote the correct world state into the residual stream, resulting in successful state transitions.
% For example, according to the results in Figure~\ref{fig:ie&probe} at $4\text{-th}$ step, we can know that in the last position $q_4$, the late-layer MLPs promote the correct world state in the residual stream. 
% Above all, given input sequences $(m_1 \ldots m_n) \in M^*$ concatenated with scratchpad $(q_1 \ldots q_{i-1}) \in M^*$, the MLPs mainly implement state transition, and the late-layer MLPs at the last position of $q_{i-1}$ promote the correct token $q_i$. 
Overall, given input sequences $(m_1 \ldots m_n) \in M^*$ concatenated with scratchpad tokens $(q_1 \ldots q_{i-1}) \in M^*$, it is primarily the MLPs—particularly those in the later layers—that implement the state transitions.

\begin{figure} [ht]
    \centering
    \includegraphics[width=\columnwidth]{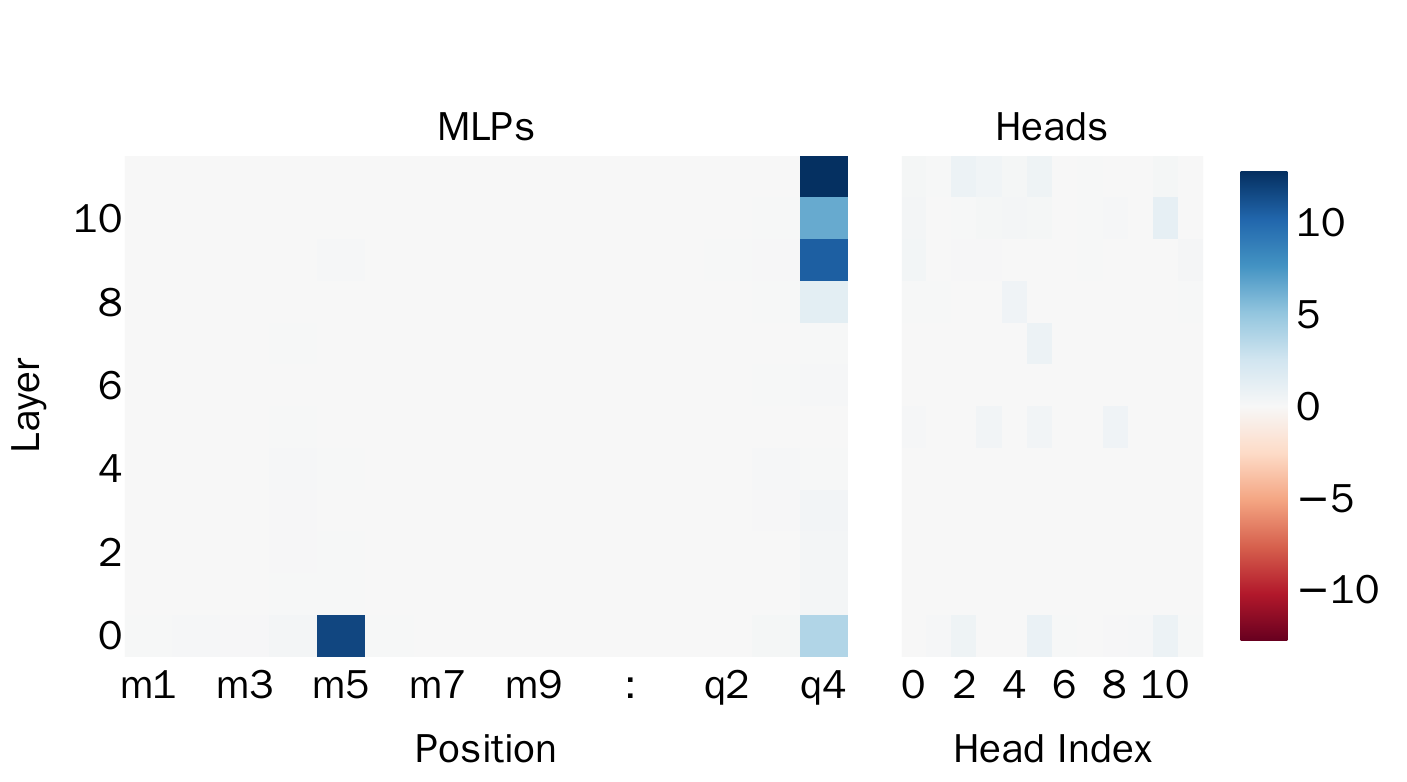}
    \caption{Activation patching results for $A_5$ word problems of length 10 at $4\text{-th}$ step, with the prompt $\mathbf{p_5} = (m_1 \ldots m_{10} | q_1 \ldots q_4)$ and the label $q_5$.
    % The depth of the color corresponds to the size of the IE value. 
    % The left and right subfigures correspond to MLPs and attention heads, respectively, highlighting that the circuit is predominantly formed by MLP0 and late-layer MLPs.
    The color intensity reflects the magnitude of the IE value. 
    The left and right subfigures represent MLPs and attention heads, respectively, emphasizing that the circuit is primarily composed of MLP0 and late-layer MLPs.}
    \label{fig:ie&probe}
\end{figure}

% \begin{figure*}[t]
%  \centering
%     \begin{subfigure}[b]{0.555\textwidth}
%         \centering
%         \includegraphics[width=\textwidth]{image/ie_v2.pdf}
%         \caption{Activation patching results in log scale.}
%         \label{fig:ie}
%     \end{subfigure}
%     \hspace{0pt} % Adjust this value to control the space between the images
%     \begin{subfigure}[b]{0.430\textwidth}
%         \centering
%         \includegraphics[width=\textwidth]{image/probe.pdf}
%         \caption{Probing results.}
%         \label{fig:probe}
%     \end{subfigure}
%     \caption {An example for A5 word problem of length 10 at $4\text{-th}$ step, with the prompt $p_5 = m_1 \ldots m_{10} | q_1 \ldots q_4$ and ground truth $q_5$.
%     (a): MLP0 at $m_5$ and late-layer MLPs at $q_4$ mainly compose the circuit. 
%     (b): Late-layer MLPs implement the state transition.
%     }
%     \label{fig:ie&probe}
% \end{figure*}

\subsection{MLP Neuron Analysis} \label{subsec: MLP Neuron Analysis}

Having identified the circuit, we then hold deeper component analysis of how late-layer MLPs implement state tracking. 
We begin by grouping all possible prompts into distinct subsets based on their resulting states.
Specifically, any prompt $(m_1 \ldots m_n | q_1 \ldots q_{i-1})$ at $i\text{-th}$ step belongs to the same subset if it reduces to the same state $q_i$ under the group operation.
For simplicity, we refer to such a subset as the $\mathcal{Q}$ prompt subset, where $\mathcal{Q}$ denotes the resulting state.
To quantify the contributions of the MLPs, we compute the contribution of each late-layer MLP to promoting $\mathcal{Q}$ in the residual stream using the logit lens \citep{geva2021transformerfeedforwardlayerskeyvalue}, evaluated across all $\mathcal{Q}$ prompt subsets.
% The results in Figure~\ref{fig:increase} show that, among late-layer MLPs, the last three layers of MLP play a significant role and MLP11 mainly implements the state transition, accounting for the 72\% logit increase.
As shown in Figure~\ref{fig:increase}, the last three layers of MLP play a significant role, with MLP11 primarily driving the state transition, contributing an average of 72\% to the logit increase.

\begin{figure}[ht]
  \includegraphics[width=\columnwidth]{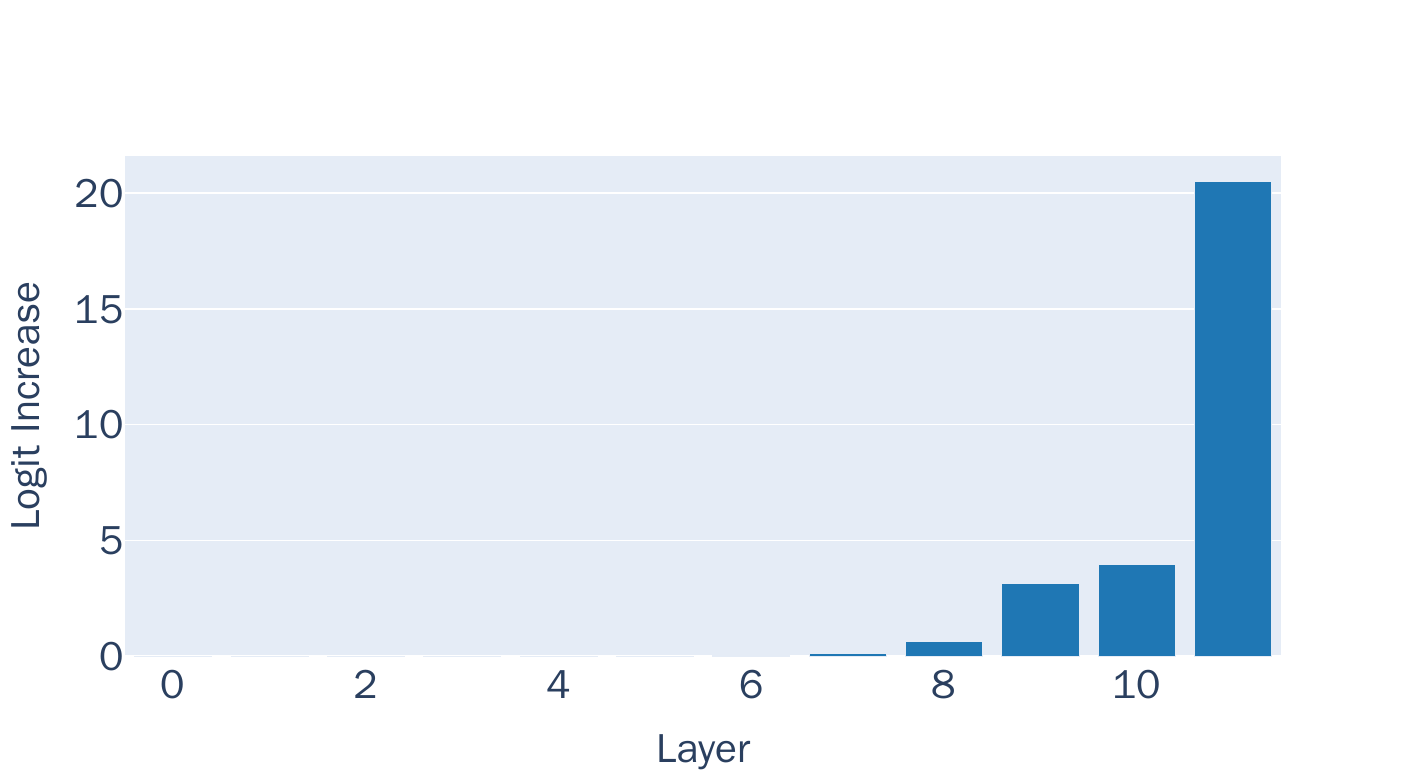}
  \caption{
  Average logit increase in the layer representation for different MLPs.
  The final three layers of the MLP primarily contribute to promoting the resulting state $\mathcal{Q}$, with MLP11 contributing the most.
  }
  \label{fig:increase}
\end{figure}

% Then we hold causal mediation analysis of late-layer MLP neurons in last position on $\mathcal{Q}$ prompt subset, but with modified intervention effects \footnote{\(\text{IE}(\textbf{q}) =  \frac{\text{P}(\textbf{q}) - \text{P}^*(\textbf{q})}{\text{P}^*(\textbf{q})}\)}, for the reason that we only consider the impact of state $\mathcal{Q}$.
% Interestingly, the intervention effect distribution has only a few prominent peaks, which means that only several top-$K$ neurons play significant role in updating the world state to $\mathcal{Q}$; 
% an example for layer $l=11$ in $\textbf{q}=0$ prompt subset is shown in Figure~\ref{fig:mlpwrtstate}.

% \citet{nikankin2024arithmeticalgorithmslanguagemodels} find that the top MLP neurons in Llama3 \citep{grattafiori2024llama3herdmodels} act as key-value memories in handling arithmetic tasks, where keys (coefficient) correspond to numerical patterns, and the associated value neurons promote correct numbers. 
% Based on this, we hypothesize that: 
% (i) in state tracking, the activation pattern corresponds to the state transition rule, and a top-$K$ neuron gets activated on a prompt subgroup, where all the prompts deduce to the same state. 
% (ii) the associated neuron promote the correct state according to the state transition rules. 

We distinguish prompts $(m_1 \ldots m_n | q_1 \ldots q_{i-1})$ according to the tuple $(m_i,q_{i-1})$\footnote{According to the transition rules, the resulting state $q_i$ is determined by the input $m_i$ and the preceding state $q_{i-1}$. Hence, each prompt can be represented as a tuple $(m_i,q_{i-1})$.}, referred to as $(m, q)$ pairs for brevity. 
We then compute the activation coefficient $m_j^l$ of the $j\text{-th}$ MLP neuron in layer $l$ across all such pairs.
% we divide the prompts into different ($m_i,q_{i-1}$) pairs (for simplicity, ($m, q$) pairs\footnote{We use $q$ to represent $q_{i-1}$.}), where $m_i$ is the $i\text{-th}$ input element, and $q_{i-1}$ is the previous world state. 
% And we analyze the activation pattern of MLP neurons in layer $l$ across all ($m, q$) pairs. 
% Interestingly, late-layer MLP neurons get activated only on a specific set of ($m, q$) pairs. 
% More specifically, the top activated 60 ($m,q$) pairs are distributed across each row or column. 
% Interestingly, according to the transition rules, there are 60 ($m,q$) pairs deducing to state $\mathcal{Q}$, distributed across each row and column. 
% Based on this, we can ask the following questions: 
% \begin{itemize}
%     \item Q1: Does the activation pattern correspond to a group of MLP neurons only activated on one $\mathcal{Q}$ prompt subset, which deduces to the state $\mathcal{Q}$?
%     \item Q2: Can we classify the MLP neurons according to the resulting states? 
% \end{itemize}
% Interestingly, late-layer MLP neurons are activated only by a specific subset of $(m, q)$ pairs. 
% More specifically, the top 60 most-activated $(m, q)$ pairs for a given neuron are distributed across individual rows or columns in the two-dimensional $(m, q)$ grid. 
% Notably, according to the transition rules, there are exactly 60 $(m, q)$ pairs that reduce to the state $\mathcal{Q}$, with these pairs also distributed across rows or columns in the grid. 
% This observation motivates the following research questions:
Late-layer MLP neurons are selectively activated by a specific subset of $(m, q)$ pairs. Specifically, the top 60 most-activated $(m, q)$ pairs for a given neuron are distributed along individual rows or columns in the two-dimensional $(m, q)$ grid. Notably, the transition rules dictate that exactly 60 $(m, q)$ pairs reduce to the state $\mathcal{Q}$, and these pairs are similarly distributed across rows or columns in the grid. This observation prompts the following research questions:
\begin{itemize}
\item \textbf{Q1:} Do the observed activation patterns indicate a group of MLP neurons that activated by one $\mathcal{Q}$ prompt subset, where the $(m, q)$ pairs reduce to the same state $\mathcal{Q}$?
\item \textbf{Q2:} Is it possible to classify MLP neurons based on the resulting states to which their activations correspond?
\end{itemize}

To address these research questions, we conduct neuron classification experiments (see Appendix~\ref{appendix algorithm}) to categorize all MLP neurons at layer $l$. 
Our findings reveal that neurons in layers MLP11, MLP10, and MLP9 can be mapped to specific states with success rates of 90.0\%, 79.6\%, and 38.7\%, respectively. 
Notably, approximately 15.5\% of MLP11 neurons achieve a perfect F1 score relative to the ground truth subset of $(m, q)$ pairs, suggesting an N-to-1 mapping between neurons and states.
% Following this classification, we compute the precision and recall of neuron activations, as described in Appendix~\ref{precision and recall}.

For Q1, most neurons exhibit high F1 scores with respect to the ground truth subset, indicating that these neurons are selectively activated by the $\mathcal{Q}$ prompt subset.  
For Q2, we systematically classify activated late-layer MLP neurons across intermediate steps, achieving high success rates. 
Notably, we observe that 15.5\% of MLP11 neurons exhibit an N-to-1 mapping to specific states.
These results suggest that transformers with chain-of-thought primarily keep track of the world state via the late-layer MLP neurons. 
These neurons exhibit selective activation in response to specific subsets of $(m, q)$ pairs and contribute to promoting the correct state token $\mathcal{Q}$ into the residual stream.  
Furthermore, accurate state transitions are achieved through the collective sub-updates of multiple neurons associated with the target state $\mathcal{Q}$.
Beyond the late-layer MLPs, we also analyze another critical component of the circuit: MLP0, which primarily achieves effective embedding of the input $m_i$, facilitating subsequent state transitions. 
Further details are provided in Appendix~\ref{appendix mlp0}.

\subsection{Transformers with Chain-of-Thought Recover FSAs } \label{subsec: automata}

Near-perfect next-token prediction does not guarantee that a generative model has fully reconstructed a world model.
For instance, in the cumulative Connect-4 task, a generative model with uniform probability can still achieve near-perfect next-token prediction \citep{vafa2024evaluatingworldmodelimplicit}.
To address this, we focus on state tracking, where world model FSAs can effectively compress distinct sequences converging to the same state while distinguishing sequences leading to different states, regardless of their superficial similarity. This subsection explores how the generative model $\tt{Transformer}_{+CoT}$, which achieves near-perfect sequence accuracy on word problems, recovers FSAs by analyzing its internal structure at the granularity of MLP neurons.

Our analysis builds on findings from Section~\ref{subsec: MLP Neuron Analysis}, which suggest that $\tt{Transformer}_{+CoT}$ performs state tracking by activating specific groups of neurons.
To evaluate the model’s ability to recover FSAs, we examine its sequence compression and distinction capabilities. For compression, we assess whether the model activates the same set of neurons for prompts converging to the same state. For distinction, we verify whether the model activates distinct neuron sets for prompts leading to different states. To quantify these properties, we propose two metrics. For two prompts, $\mathbf{p_i}$ and $\mathbf{p_i'}$, that converge to the same state, we define the compression metric as:
\[
\text{Compression} = \frac{ |\text{N}_{\mathbf{p_i}} \cap \text{N}_{\mathbf{p_i'}}| }{|\text{N}_{\mathbf{p_i}} \cup \text{N}_{\mathbf{p_i'}}|}
\]
where $\text{N}_{\mathbf{p_i}}$ and $\text{N}_{\mathbf{p_i'}}$ denote the sets of neurons activated by prompts $\mathbf{p_i}$ and $\mathbf{p_i'}$, respectively.
Conversely, for prompts $\mathbf{p_i}$ and $\mathbf{p_i'}$ leading to distinct states, we define the distinction metric as:
\[
\text{Distinction} = 1 - \frac{ |\text{N}_{\mathbf{p_i}} \cap \text{N}_{\mathbf{p_i'}}| }{|\text{N}_{\mathbf{p_i}} \cup \text{N}_{\mathbf{p_i'}}|}
\]

We calculate compression and distinction metrics pairwise for all prompts, averaging them at each intermediate step.
Our results demonstrate that $\tt{Transformer}_{+CoT}$ effectively compresses prompts sampled from the $\mathcal{Q}$ prompt subset while distinguishing prompts from different subsets, achieving near-perfect world model recovery with an average metric score of 0.991.
This indicates efficient capture of the underlying FSA structure.
Extending our analysis beyond the group $A_5$, we identify similar FSAs for $\mathbb{Z}_{60}$ and $A_4 \times \mathbb{Z}_5$, as well as in larger-scale models, with detailed results in Appendix~\ref{appendix other groups}. 
Collectively, we conclude that \textbf{Transformers with Chain-of-Thought recover FSAs on state tracking, even at the granularity of MLP neurons}.

\section{Robustness and Generalizability of FSAs} \label{sec: robustness}

Sections~\ref{sec:model performance} and~\ref{sec:mechanism} have demonstrated positive results in the recovery of FSAs.
However, the word problems considered are too idealistic, with sequence length controlled and noise excluded.
In this section, we will investigate the robustness and generalizability of the FSAs within $\tt{Transformer}_{+CoT}$, considering intermediate step skipping, scratchpad noise, and length generalization.
We find that refining the training data distribution can steer the model toward better adaptation to the aforementioned scenarios, except for length generalization.
Addressing the latter may require exploring alternative approaches, such as modifying the model architecture, improving training strategies, etc.

\subsection{Skipping Steps} \label{Skipping}

% In practical training data, intermediate step jumps are very common.
In practical training datasets, such as OpenWebMath\footnote{\url{https://github.com/keirp/OpenWebMath}}, reasoning trajectories include implicit or skipped intermediate steps. 
For example, the length of $(q_1 \ldots q_n)$ may be less than $n$, with some intermediate steps skipped.
In our experiment settup, we introduce a skipping probability to omit each intermediate state \( q_i \) (except \( q_n \)) and train $\tt{Transformer}_{+CoT}$ on this dataset. 
To investigate the model’s ability to perform intermediate step jumps, we employ a linear classifier \citep{gurnee2023findingneuronshaystackcase} to probe the states embedded in the residual stream at the last position.
Specifically, we probe the token $q_i$ at the last position, given the prompt $(m_1 \ldots m_n | q_1 \ldots q_{i-1})$.
From the results in Figure~\ref{fig:multi_step_probe}, we can find that $q_i$, $q_{i+1}$, and $q_{i+2}$ are all embedded in the residual stream, suggesting that the model has learned single-step reasoning, two-step reasoning (skipping one step), as well as three-step reasoning (skipping two steps).
% And we find that the index of layers with sufficiently high probing results for $q_{i+2}$ is greater than that for $q_{i+1}$, which is greater than that for $q_{i}$ in turn.
Additionally, we propose possible mechanisms for skipping and design experiments for further analysis, as detailed in appendix~\ref{appendix skipping}.
Finally, we conclude that incorporating skipping examples into the training set facilitates the adaptation of the FSAs within $\tt{Transformer}_{+CoT}$ to the scenario.
% In terms of skipping two steps, we stress that in the distribution of the training data set, there is a probability of only 0.038 that there is a two-step skipping (refer to Appendix~\ref{sec:appendix} for calculation details).
% A small percentage of the training data has taught the model to learn two-step reasoning.

\begin{figure}[ht]
  \includegraphics[width=\columnwidth]{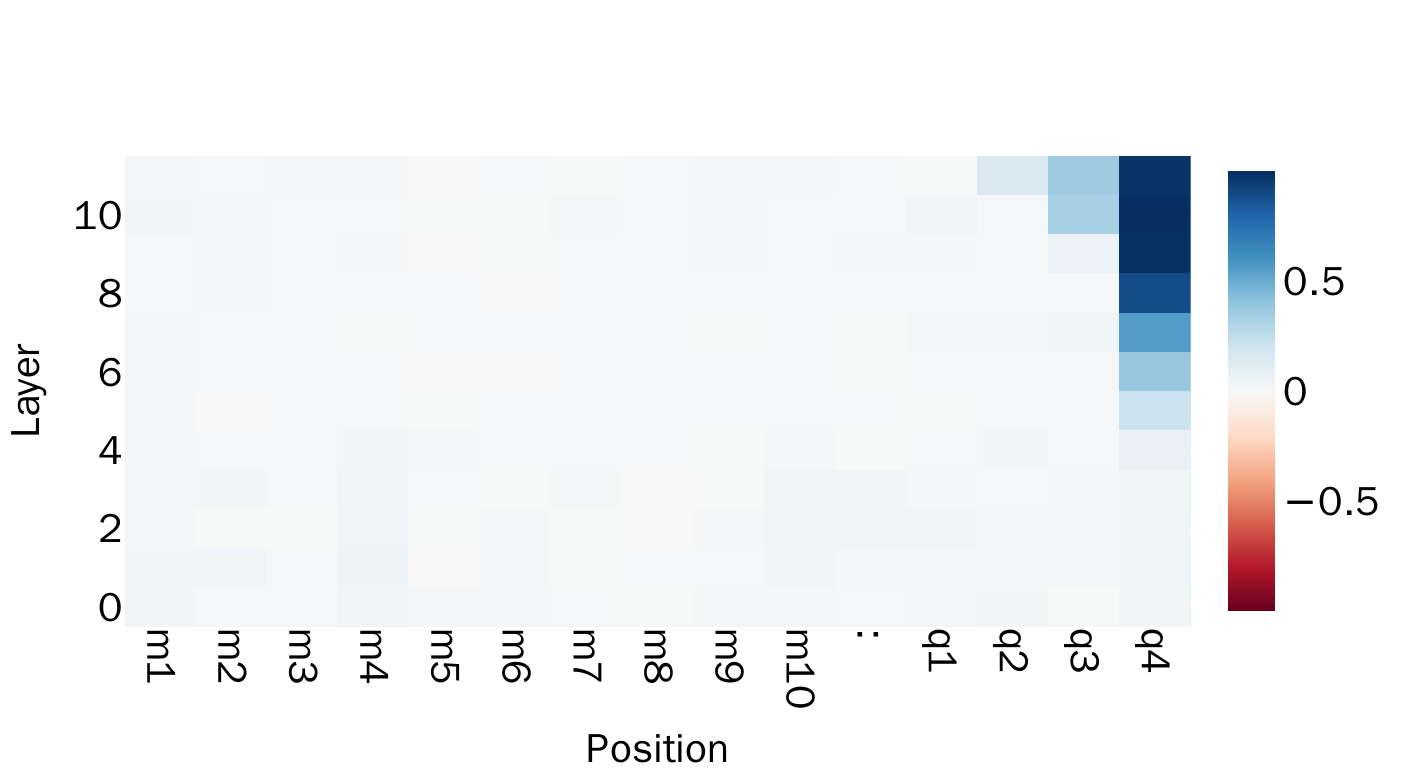}
  \caption{
  Probing results for state $q_i$ at $4\text{-th}$ step.
  We can find the positive results appearing two positions earlier, demonstrating that the model has learned skip-step reasoning, including the ability to skip one or two intermediate steps.
  }
  \label{fig:multi_step_probe}
\end{figure}

\subsection{Noise in Scratchpad} \label{subsec: noise}

We analyze the impact of injecting noise into the scratchpad states $(q_1 \ldots q_{i-1})$, where at least one state is false. 
We categorize noise into two types: (1) any false state except $q_{i-1}$, and (2) a false $q_{i-1}$. 
We evaluate the accuracy of model's next-token prediction under these conditions.
For the first type of noise, where noise affects any state other than $q_{i-1}$, the model maintains near-perfect next-token prediction accuracy. 
The reason is obvious: the model depends primarily on the input $m_i$ and the preceding state $q_{i-1}$ for the state transition at the $i$-th step, while disregarding intermediate states prior to $q_{i-1}$.
For the second type of noise, where $q_{i-1}$ is false, the model’s performance declines significantly, with next-token prediction accuracy falling to approximately 0.017.

% In this section, we consider injecting noise into scratchpad $q_1 \ldots q_{i_1}$, where some previous state is false.
% We divide the noise into the following two types: any false state except $q_{i-1}$ and false $q_{i-1}$.
% And we test the accuracy of the model in predicting the next token with noise in scratchpad (we use token accuracy as metric here).
% For the first type noise, the model still achieves almost perfect next-token prediction.
% The reason is obvious: when the model implements state transition at the $i\text{-th}$ step, it mainly focuses on $m_i$ and the previous state $q_{i-1}$, while ignoring intermediate states prior to $q_{i-1}$.
% And for the second type of noise, the model's performance dramatically decreases to an average accuracy of nearly zero (0.017).

Here, we aim to address the scenario by adjusting the dataset distribution.
Specifically, we train the model using data containing the second type of noise in the scratchpad states and achieve a substantial improvement in robustness: the accuracy increases from 0.017 to 0.896.
To investigate whether the model accurately recovers the correct state tracking from noise, we employ a linear classifier \citep{gurnee2023findingneuronshaystackcase} at the final position.
Specifically, given a corrupted prompt $(m_1 \ldots m_n | q_1 \ldots \hat{q}_{i-1})$ (where $\hat{q}_{i-1}$ denotes an incorrect state), we probe the correct next state $q_{i}$, the false previous state $\hat{q}_{i-1}$ and the true previous state $q_{i-1}$.
We find that the activation sustains positive probing results for $\hat{q}_{i-1}$ across layers until the final one, where it demonstrates a high average accuracy of 0.943 for $q_{i}$ and 0.379 for $q_{i-1}$.
% Based on this, we propose one possible mechanism that the model attends to intermediate representation at previous step, thus retrieving the correct input state $q_{i-1}$ and updating the state to $q_{i}$ successfully.
In autoregressive models, the residual stream at the preceding position of $q_{i-2}$ in the sequence $(m_1 \ldots m_n | q_1 \ldots \hat{q}_{i-1})$ encodes the correct input state $q_{i-1}$.
Based on this, we propose a possible mechanism that the model's attention heads attend to the activation at the preceding position, thus retrieving the correct input state $q_{i-1}$ and consequently updating the state to $q_{i}$.
We have also designed experiments to analyze the hypothesis, as detailed in Appendix~\ref{appendix noise}.
% According to the probing results, we propose one possible mechanism that at $i\text{-th}$ step, the model suppresses the wrong input state $\hat{q_{i-1}}$, and attends to the intermediate representation at previous step, thus retrieving the correct input state $q_{i-1}$ and updating the state to $q_{i}$ successfully.
% We validate this by masking the position ${q_{i-2}}$ in the last layer and evaluate the probing accuracy change for state $q_{i}$ in the last position.
% The results indicate that after masking, the MLP11 probing results dropped from greater than 0.9 to nearly zero.
% Moreover, the model has also implement correct state transition this step, which means the model has implemented two-step reasoning like Section~\ref{Skipping}.
We conclude that refining dataset distribution can help in maintaining the accuracy and reliability of the FSAs inside $\tt{Transformer}_{+CoT}$.

\subsection{Length Generalization}

Lastly, we explore length generalization, a key challenge in developing transformer-based LLMs.
To enhance length generalization, we replace GPT2's original absolute positional embedding with NoPE (No Positional Embedding), which has been shown in prior work \citep{10.5555/3666122.3667204} to outperform other explicit positional embedding methods, including both absolute and relative approaches.
Notably, in our experiments, NoPE exhibits a certain degree of length generalization, whereas absolute and relative positional embeddings show negligible generalization.
We choose LSTM for comparison for the reason that it efficiently learns all investigated word problems on sequences of arbitrary length. 

Table~\ref{tab:ood} illustrates the length generalization capabilities of various models on unseen sequence lengths during training.
LSTM achieves perfect generalization, while transformers with chain-of-thought exhibit weak performance. 
Analyzing the failure from the perspective of MLP neurons, we identify an intriguing ``U-turn'' phenomenon: the precision of activated neurons, which, after initially decreasing, exhibits an increase during the final few steps of inference.
Further details are provided in Appendix~\ref{appendix length}.
We conclude that improving the length generalization ability of the model requires other approches, such as modifications to model architecture, loss functions, or optimization algorithms.

\begin{table}[ht]
\centering
\resizebox{\columnwidth}{!}{
\begin{tabular}{lccccccc}
\toprule
       & \multicolumn{7}{c}{length} \\
       \cmidrule(lr){2-8} % Adds a rule under "length" spanning columns 2 to 8
Models & 20 & 21 & 22 & 23 & 24 & 25 & 30 \\
\midrule
LSTM & 1.00 & 1.00 & 1.00 & 1.00 & 1.00 & 1.00 & 1.00  \\
$\tt{Transformer}_{+CoT}$ & 0.99 & 0.98 & 0.87 & 0.53 & 0.24 & 0.09 & 0.00 \\
\bottomrule
\end{tabular}
}
\caption{Comparison of length generalization performance between LSTM and $\tt{Transformer}_{+CoT}$. 
Models are trained on sequences of mixed lengths up to 20 and tested on seen (20) and unseen (>20) lengths.}
\label{tab:ood}
\end{table}

\section{Related Work}

Theoretical studies \citep{merrill2024illusionstatestatespacemodels, liu2023transformerslearnshortcutsautomata} integrate circuit complexity and algebraic formal language theory to categorize different models and word problems into distinct computational complexity classes, such as $\mathsf{TC}^0$ and $\mathsf{NC}^1$. In particular, \citet{merrill2024illusionstatestatespacemodels} demonstrate that $\mathsf{NC}^1$-complete word problems cannot be expressed by models within the $\mathsf{TC}^0$ class. In contrast, our work investigates $\tt{Transformer}_{+CoT}$ through the lens of mechanistic interpretability, revealing a dual advantage in expressiveness and learnability, while also uncovering FSA structures within the model.

% \citet{vafa2024evaluatingworldmodelimplicit} suggest that next-token prediction is a fragile metric for evaluating generative models and introduce novel evaluation metrics for world model recovery, inspired by Myhill-Nerode boundary theory.
% These metrics—such as distinction recall and precision—assess how effectively the model distinguishes between prefixes that lead to different states.
% Distinct from the distinction precision and recall metrics proposed by \citet{vafa2024evaluatingworldmodelimplicit}, our metric operates at the level of individual MLP neurons, assessing how the model internally differentiates between input sequences. This provides a more fine-grained, mechanistic evaluation of the model’s internal dynamics. Nevertheless, both approaches are fundamentally grounded in Myhill-Nerode equivalence classes.
\citet{vafa2024evaluatingworldmodelimplicit} argue that next-token prediction is a fragile metric for evaluating generative models and propose novel evaluation metrics for world model recovery, drawing inspiration from Myhill-Nerode boundary theory \citep{myhill1957finite, nerode1958linear}. Their metrics—such as distinction recall and precision—evaluate how effectively a model distinguishes between prefixes that lead to different states. In contrast, our proposed metric operates at the level of individual MLP neurons, measuring how the model internally differentiates between input sequences. This offers a more fine-grained, mechanistic perspective on the model’s internal dynamics.

\section{Conclusions}
In this work, we investigate the learnability of $\tt{Transformer}_{+CoT}$ through the lens of mechanistic interpretability in state tracking. We begin by reaffirming the success of $\tt{Transformer}_{+CoT}$ in both in-distribution and out-of-distribution settings. Next, we identify the key components of the circuit, specifically the neurons in the last layers of the MLP, which play a critical role. We find evidence of an implicit FSA within the model, where each state is compressed by a distinct group of neurons. 
Finally, we evaluate the model in three challenging settings and show that the learned FSAs are robust to noise and skipping but struggle with generalization to significantly longer sequences. This suggests the need for architectural or optimization improvements. 
Our findings reveal that $\tt{Transformer}_{+CoT}$ achieves near-perfect performance in next-token prediction while internalizing an FSA-like structured state representation, bridging the gap between expressiveness and learnability. This work provides insights into structured reasoning for sequential tasks.

\section{Acknowledgments}

This research is supported by the National Natural Science Foundation of China under Grant Nos. 62436006, 62192731, 62192730.
Jie Fu is supported by Shanghai Artificial Intelligence Laboratory. 

% Despite these promising results, our study is limited to GPT2-style architectures. Future research should explore more diverse models, including Mixture-of-Experts (MoE) architectures and advanced memory-augmented Transformers. Additionally, disentangling state tracking from other emergent capabilities remains a challenge in practical applications.

\section*{Limitations}

In this paper, we define state tracking as word problems involving cyclic and symmetric groups, which form the basis of our conclusions. Our study focuses solely on GPT2 models, leaving the exploration of additional models, such as Llama, for future research. Furthermore, while we have made efforts to address more realistic word problems, state tracking in real-world tasks is intricate, making it challenging to separate state tracing abilities from other skills.

\bibliography{acl_latex}

\appendix

\section{Symbol Description} \label{sec:appendix}

In this section, we briefly introduce the symbols used in this article as in Table~\ref{tab:symbols}.

\begin{table*}
    \centering
    \begin{tabular}{lp{12cm}}
\toprule
\textbf{Symbol} & \textbf{Description} \\ \midrule
$\Sigma$               & Finite set of characters                     \\
$Q$               & Finite set of states in the automaton             \\
$q_0$               & Start state, an element of the state set Q      \\
$\delta$               & State-transition function                  \\
$(M, \cdot)  $             & A monoid                    \\
$M^*$               & Set of all possible sequences from $M$      \\
$e$               & The starting state                    \\
$\mathbb{Z}_k$               & Cyclic group of order $k$             \\ 
$S_k$              & Symmetric group of order $k$           \\ 
$\mathsf{TC}^0$            &  A complexity class in computational complexity theory, consisting of problems solvable by uniform constant-depth threshold circuits with a polynomial number of gates.     \\
$\mathsf{NC}^1$            &  A complexity class containing problems solvable by uniform logarithmic-depth Boolean circuits with a polynomial number of gates and bounded fan-in.     \\
$\mathbb{Z}_{60}$          & An abelian group encoding mod-60 addition      \\
$A_5$ & The alternating group on five elements \\
$A_4 \times \mathbb{Z}_5$ & A non-abelian but solvable group \\
$\mathbf{p_i}$ & Prompt at the $i$\-th step \\
$q_i$ & Resulting state at the $i$\-th step \\
$\mathcal{Q}$ & Resulting state $q_i$ \\
$m_i$ & The $i$\-th input of word problems \\
$q_{i-1}$ & The preceding state of word problems \\
$(m,q)$ & $(m_i, q_{i-1})$ \\
$\hat{q}_i$ & Error state \\
$\mathbf{p_i'}$ & Counterfactual prompt at the $i$\-th step \\
$q_i'$ & Counterfactual result state at the $i$\-th step \\
$\text{P}$ & Pre-intervention probability distribution\\
$\text{P}^*$ & Post-intervention probability distribution\\
$\mathbf{x^l}$ & Input at layer $l$\\
$\mathbf{K^l}$ & The weight matrix at layer $l$\\
$\mathbf{V^l}$ & The weight matrix at layer $l$ \\
$d_{\text{mlp}}$ & the MLP intermediate dimension \\
$d_{\text{m}}$ & the model dimension \\
$f$ & Non-linear activation function \\
$\mathbf{k^l_j}$ & the $j\text{-th}$ row vectors of $\mathbf{K^l}$ \\
$\mathbf{v^l_j}$ & the $j\text{-th}$ row vectors of $\mathbf{V^l}$ \\
$m_{j}^l$ & Activation coefficient at the $j$\-th neuron \\
$\text{N}_{\mathbf{p_i}}$ & Set of activated neurons on prompt $\mathbf{p_i}$ \\
$\text{N}_{\mathbf{p_i'}}$ & Set of activated neurons on prompt $\mathbf{p_i'}$ \\
\bottomrule
\end{tabular}
\caption{Symbols used in this paper}
\label{tab:symbols}
\end{table*}

% 342行
\section{Separating Elements in Groups} \label{groups}

In Section~\ref{subsec: ood}, we separate all $m$ into proper subsets to explore the models' performance.
As Table~\ref{tab: division} shows, we divide $m$ in $\mathbb{Z}_{60}$, $A4 \times \mathbb{Z}_5$ and $A_5$ according to values, orders and permutation types, respectively.
We then prove that the division is reasonable, for the reason that every proper subset can generate the whole group.

\begin{theorem}
    Every separated proper subset can generate the whole group under group operation.
\end{theorem}

\begin{proof}

    For Cyclic group, every element is a generator, so that proper subsets can generate \( h \in \mathbb{Z}_{60} \).
    
    For alternating group, the 3-cycle and 5-cycle in \( A_5 \) can both generate the group.

    For direct product group, we continue the proof as follows:

    The elements of order 15 in the direct product group \( A_4 \times \mathbb{Z}_5 \) are those elements of the form \( (g, h) \), where \( g \in A_4 \) is a 3-cycle (order 3) and \( h \in \mathbb{Z}_5 \) is a nonzero element (order 5).
    Although these elements of order 15 form a proper subset of the entire group, their projections onto the factors \( A_4 \) and \( \mathbb{Z}_5 \) separately generate the full groups:
    \begin{itemize}
        \item The 3-cycle in \( A_4 \) can generate \( A_4 \) (since \( A_4 \) is generated by its 3-cycles).
        \item Any nonzero element in \( \mathbb{Z}_5 \) is a generator, thus can generate the entire \( \mathbb{Z}_5 \).
    \end{itemize}
    By the generation theorem of direct product groups, if a subset has surjective projections onto each factor, then the group it generates must be the entire direct product.
    Thus, all elements of order 15 can generate the entire group \( A_4 \times \mathbb{Z}_5 \).
\end{proof}

We train the model with sequences sampled from one subset, and evaluate on sequences with $m$ sampled from the group.
We emphasize that although the model did not encounter mixed sequences during training, the proper subset is able to generate the entire group. 
Therefore, all possible state transitions are included in the data.
$\tt{Transformer}_{+CoT}$ success in out-of-distribution shows that the learned algorithm does not rely on finding a certain pattern in the input sequence.

% Surprisingly, LSTM fails to generalize as an encoder, and we have also explored its performance as a decoder, where LSTM fails to converge, however.
Surprisingly, LSTM fails to generalize in the task.
% We The failure of the LSTM may be due to its design as an encoder, while CoT can include all elements of the group within the entire sequence, by concatenating every possible states after the input sequence.
We hypothesize that the disparity in out-of-distribution performance stems from the difference in task formulation: 
the LM task can integrate all elements of the group into the sequence by concatenating each possible state after the input sequence.

\begin{table*}
  \centering
  \begin{tabular}{lp{10cm}}  
    \hline
    \textbf{Group} & \textbf{Proper Subsets} \\
    \hline
    $A_{5}$ & \textbf{Identity and Double Transpositions:} \\  
            & 0, 3, 8, 11, 12, 13, 14, 27, 30, 33, 41, 43, 47, 53, 55, 59 \tabularnewline
            & \textbf{3-Cycles:} \\  
            & 1, 2, 4, 5, 6, 7, 9, 10, 15, 19, 22, 24, 28, 29, 37, 39, 40, 49, 51, 52 \tabularnewline
            & \textbf{5-Cycles:} \\  
            & 16, 17, 18, 20, 21, 23, 25, 26, 31, 32, 34, 35, 36, 38, 42, 44, 45, 46, 48, 50, 54, 56, 57, 58 \\
    \hline
    $A4 \times \mathbb{Z}_5$ & \textbf{Order 15:} \\  
            & 6, 7, 8, 9, 11, 12, 13, 14, 21, 22, 23, 24, 26, 27, 28, 29, 31, 32, 33, 34, 36, 37, 38, 39, 46, 47, 48, 49, 51, 52, 53, 54 \tabularnewline
            & \textbf{Other Orders:} \\  
            & 0, 15, 40, 55, 5, 10, 20, 25, 30, 35, 45, 50, 1, 2, 3, 4, 16, 17, 18, 19, 41, 42, 43, 44, 56, 57, 58, 59 \\
    \hline
    $\mathbb{Z}_{60}$ & \textbf{$<30$}: \\  
            & 0, 1, 2, 3, 4, 5, 6, 7, 8, 9, 10, 11, 12, 13, 14, 15, 16, 17, 18, 19, 20, 21, 22, 23, 24, 25, 26, 27, 28, 29 \tabularnewline
            & \textbf{$\ge 30$}: \\  
            & 30, 31, 32, 33, 34, 35, 36, 37, 38, 39, 40, 41, 42, 43, 44, 45, 46, 47, 48, 49, 50, 51, 52, 53, 54, 55, 56, 57, 58, 59 \\
    \hline
  \end{tabular}
  \caption{Division of elements for three groups in out-of-distribution.}
  \label{tab: division}
\end{table*}

\section{More Experimental Details}

In Section~\ref{sec:mechanism}, we train a GPT2-small model with 12 layers, 12 heads, and a hidden dimension of 768 on a synthetic group dataset. Specifically, we use 2 NVIDIA A100 GPUs, employing mixed-precision bfloat16 for improved computational efficiency. For one group experiment, each run takes approximately 50 minutes. We configure the AdamW optimizer with a weight decay of 0.01 and momentum parameters $\beta = (0.9, 0.999)$, alongside a constant learning rate of 0.0001 and a batch size of 512. The GPT2-small model is trained on a mixture of word problems with $n$ ranging from 2 to 20, and optimized via gradient-based methods with a cross-entropy loss function. Early stopping is implemented once the test accuracy reaches 99\%, after which the model is used for MI experiments.

% 469行
\section{Classification Algorithms} \label{appendix algorithm}

We provide the late-layer MLP neurons classification procedures as follows:

\begin{enumerate}[noitemsep]
    \item Utilize the priori transition rules to pre-compute all ground truth sets of $(m,q)$ pairs. 
    \item Measure all activation coefficient $m_{j}^l$ across all ($m, q$) pairs. 
    \item Utilize the logit lens to calculate the logits of tokens embedded in $\mathbf{v_j^l}$, and convert to a 2D pattern, where the cell in index ($m, q$) is the logit of the result state $\mathcal{Q} = m \cdot q$. 
    \item Multiply the intermediate results of the previous two steps element-wise, resulting in effective logit contribution of the neuron to the corresponding state for each ($m, q$) pair. 
    \item  Extract the top 60 ($m, q$) pairs from the activation pattern as the prediction. 
    \item Compute the F1 scores between prediction and all labels. 
    The neuron can be classified to state with F1 score no less than threshold $\theta=0.2$.\footnote{Here the value of $\theta$ is an empirical setting, significantly higher than the random value of 0.017. See Appendix~\ref{F1} for details on computing the random value.}
\end{enumerate}

% \begin{algorithm}
% \caption{Classify MLP Neurons at Layer $l$}
% \begin{algorithmic}[1]
% \State \textbf{Input:} Prior transition rules $T$, activation coefficients matrix $M_j^l$ for all $(m, q)$ pairs, value vector $v_j^l$, threshold $\theta$
% \State \textbf{Output:} Classified neurons
% \State Compute all labels $L$ from $T$
% \For{each neuron $j$ in layer $l$}
%     \State Compute logit matrix $Z_j$ for all $(m, q)$ pairs from $v_j^l$ using logit lens
%     \State Compute effective logit contribution $C_j = M_j^l \odot Z_j$
%     \State Select top 60 $(m, q)$ pairs from $C_j$ as prediction $P_j$
%     \State Compute F1 score between $P_j$ and $L$
%     \If{F1 score $\geq \theta$} 
%         \State Classify neuron $j$ as relevant
%     \Else
%         \State Classify neuron $j$ as irrelevant
%     \EndIf
% \EndFor
% \State \Return Classified neurons
% \end{algorithmic}
% \label{al: classify}
% \end{algorithm}

% 477行
\section{Expected F1 Score of Randomly Activated Neurons} \label{F1}

In Section~\ref{subsec: MLP Neuron Analysis}, we empirically set the value of $\theta$ as 0.2 to classify MLP neurons, which is non-trivial compared with random F1 score 0.017.
And we provide detailed calculation process of randomly activated neurons' F1 score as follows.

Let \( S \) be a set containing 3600 elements: \(S = \{0, 1, 2, \ldots, 3599\}\), and let \( L \) be the label subset consisting of the first 60 elements: \(L = \{0, 1, 2, \ldots, 59\}\).
We randomly sample 60 elements from \( S \) to form a subset \( x \): \(x \subseteq S\) and \(|x| = 60\).
Our goal is to compute the expected values of precision, recall, and F1 score for the subset \( x \) with respect to the label subset \( L \).

Define \(T = |x \cap L|\) as the number of correctly selected elements. 
Since \( x \) is chosen uniformly at random from \( S \), and given that \( L \) contains 60 elements, \( T \) follows a hypergeometric distribution. 
Its expected value is:
\[
\mathbb{E}[T] = \frac{|x| \cdot |L|}{|S|} = \frac{60 \times 60}{3600} = 1.
\]
The precision and recall are defined by
\[
\text{Precision} = \frac{|x \cap L|}{|x|} = \frac{T}{60} 
\]
\[ \text{Recall} = \frac{|x \cap L|}{|L|} = \frac{T}{60}
\]
Thus, their expected values are
\[
\mathbb{E}[\text{Precision}] = \mathbb{E}\left[\frac{T}{60}\right] = \frac{\mathbb{E}[T]}{60} = \frac{1}{60},
\]
\[
\mathbb{E}[\text{Recall}] = \frac{1}{60}.
\]
The F1 score is the harmonic mean of precision and recall:
\[
\text{F1} = \frac{2 \cdot \text{Precision} \cdot \text{Recall}}{\text{Precision} + \text{Recall}}.
\]
Since \(\text{Precision} = \text{Recall} = \frac{T}{60}\), we substitute to obtain
\[
\text{F1} = \frac{2 \cdot \frac{T}{60} \cdot \frac{T}{60}}{\frac{T}{60} + \frac{T}{60}} = \frac{2 \cdot \frac{T^2}{3600}}{\frac{2T}{60}} = \frac{T}{60}.
\]
Hence, the expected F1 score is
\[
\mathbb{E}[\text{F1}] = \mathbb{E}\left[\frac{T}{60}\right] = \frac{1}{60} \approx 0.01667.
\]

% 549行
\section{Activated Neurons Analysis} \label{precision and recall}

Having classified neurons into states as described in Section~\ref{subsec: MLP Neuron Analysis}, We can compute the precision and recall of the activated neurons, averaged over all $\mathcal{Q}$ prompt subsets, during the model's single-step state transitions.
Specifically, we select the top-$K$ activated neurons based on the activation coefficient $m_{j}^l$ and calculate precision and recall using the classified neurons as ground truth labels.
The results presented in Figure~\ref{fig:p&r} indicate that across intermediate steps, the model activates MLP11 neurons with high average precision (0.797) but relatively low average recall (0.253).
We emphasize that the high precision accounts for the model’s accuracy in solving word problems, whereas the low recall suggests that the neurons activated by a given $\mathcal{Q}$ prompt subset can vary across different steps.

\begin{figure}[ht]
  \includegraphics[width=\columnwidth]{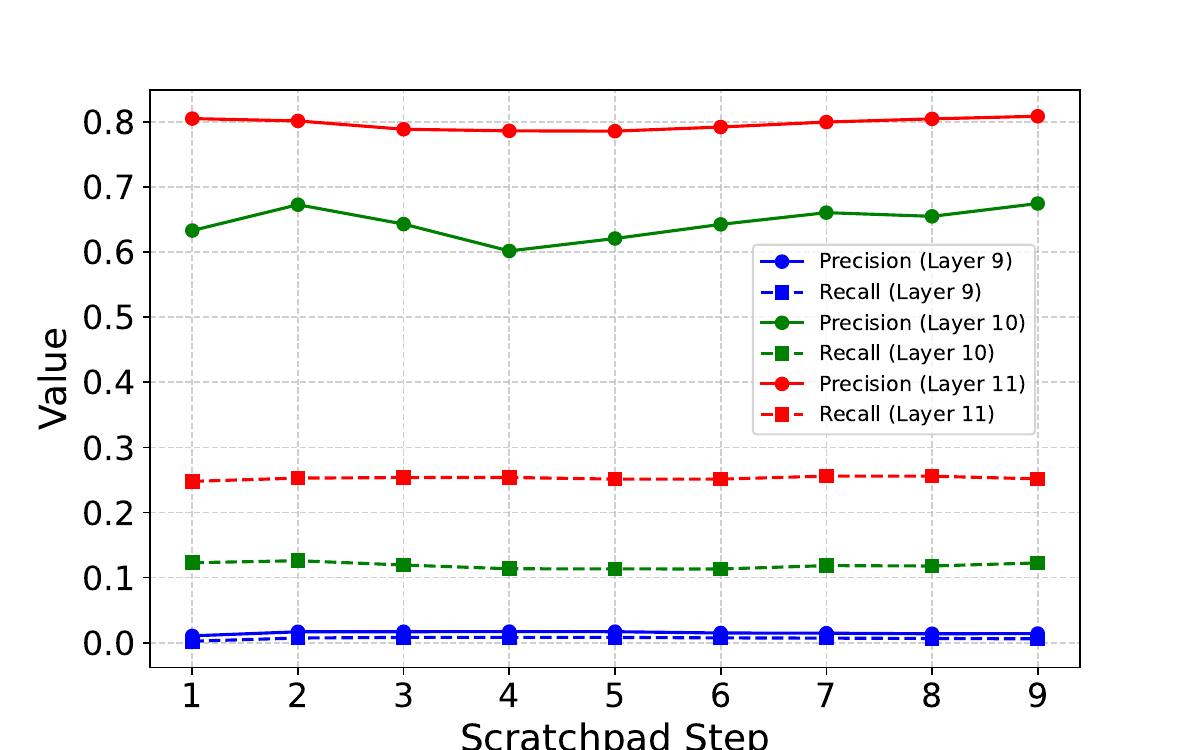}
  \caption{
  Averaged precision and recall of top-$K$ activated neurons across different intermediate steps.
  }
  \label{fig:p&r}
\end{figure}

% 560行
\section{MLP0 in Circuit} \label{appendix mlp0}

Activation patching results in Section~\ref{subsec: circuit} show that the circuit mainly consists of MLP0 and late-layer MLPs.
Given a prompt $(m_1 \ldots m_n | q_1 \ldots q_{i-1})$, the late-layer MLPs implement state transition in the postion $q_{i-1}$, while the MLP0 mainly achieves effective word embedding at the position of $m_{i}$.
As Figure\ref{fig:mlp0} shows, the MLP0 promotes input $m_i$ into the residual stream, which will then be transferred to the last position of $q_{i-1}$ by attention heads for subsequent state transitions.

\begin{figure}[ht]

\includegraphics[width=\columnwidth]{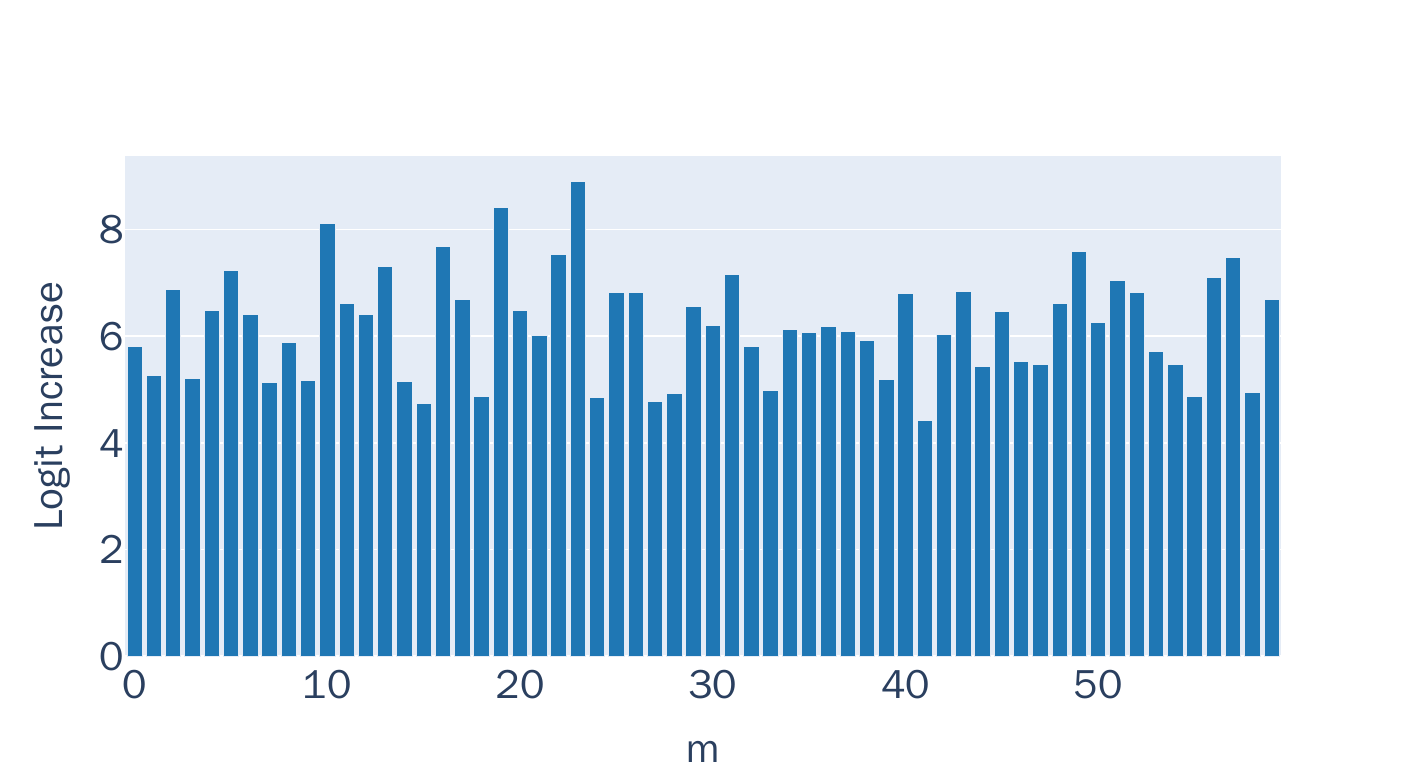}
  \caption{
  Average input $m$ logit increase in the layer representation for MLP0.
  The MLP0 mainly implements effective word embedding in the circuit.
  }
  \label{fig:mlp0}
  
\end{figure}

\section{Attention Heads} \label{appendix attention}

The activation patching results in Section~\ref{subsec: circuit} reveal that the circuit is primarily composed of MLPs, which predominantly encode state transitions. 
Additionally, attention heads within the circuit facilitate the flow of information across positions. 
We identify the principal attention heads using the activation patching results shown in Figure~\ref{fig:ie&probe} and analyze their activation patterns. 
The findings, presented in Figure~\ref{fig: attention}, reveal two distinct attention patterns:
\begin{enumerate}[noitemsep]
\item Across all steps, specific attention heads directly transfer information from the position of $m_i$ to the position of $q_{i-1}$.
\item The model develops a secondary pattern: information from positions $m_1 \ldots m_n$ is relayed backward to the position of ``:'', which separates the input and scratchpad and marks the start of state tracking, thus enabling the model to attend to this relay position during inference.
\end{enumerate}

% \begin{figure}[ht]
%   \includegraphics[width=1.05\columnwidth]{image/attention_v3.pdf}
%   \caption{
%   Two attention patterns.
%   }
%   \label{fig:attention}
% \end{figure}

% In terms of the attention patterns in Section~\ref{subsec: attention}, a natural question is how the attention heads extract $m_i$ from the numerous input $m_1 \ldots m_n$ embedded in the residual stream at the relay position.

% We hypothesize that this extraction depends on the first type of attention pattern, which has injected the input information $m_i$ into the current position, facilitating the second type attention heads to differentiate the mixed information embedded in relay position.
% Furthermore, we validate this hypothesis by masking position $m_i$ at the $q_{i-1}$ position and probing the changes of retrieving input $m_i$ at the current step.
% When there is no mask, the average accuracy of probing $m_i$ in the representation of the middle and later layers, except for the last layer (high results for $q_i$), is greater than 0.9. 
% However, when the mask is applied, it drops to nearly 0.
% We conclude that, the model uses two interrelated dynamic attention patterns to focus on the correct input $m_i$ at $i\text{-th}$ step, corresponding to the FSA accepting each step of the input.

% 580行
\section{Automta still exists on other groups with large model scale.}  \label{appendix other groups}

We have found a FSA within GPT2-small on $A_5$ as in Section~\ref{subsec: automata}, in this section, we will extrapolate this conclusion to other groups and larger model size.
We first conduct analogous experiments on $A_4 \times \mathbb{Z}_5$ and $ \mathbb{Z}_{60}$.
We find that MLP11 neurons can still be classified according to transition rules, and the model retains the ability to compress and distinguish different input sequences, achieving average compression and distinction metrics of 0.992 and 0.997, respectively.
Moreover, we explore whether FSAs still exist with model size larger.
Specifically, we repeat the experiments, increasing the model size from GPT2-small (124M) to GPT2-large (744M), and find FSAs still exist.
Results are shown in Figure~\ref{fig: stillexists}.

% 607行

% \section{Skipping} \label{appendix skipping}

% There are two possible mechanisms for skipping in Transformer+CoT: one is to use continuous layers to achieve skip-step state transitions, and the other is to use a single layer to achieve skip-step transitions.
% The key difference is that the former depends on the skipped states, while the latter does not.
% To explore what algorithms the model uses in skipping, we design intervention experiments by suppressing skipped tokens.
% Specifically, given the prompt $m_1 \ldots m_n | q_1 \ldots q_{i-1}$, we probe the token $q_{i+1}$ in the last position, while suppressing token $q_{i}$.
% We find that suppressing $q_{i}$ has caused the positive probing results to drop almost to zero in the last position with $\sigma$ ranging between 0 and 1, suggesting that skipping necessitates the skipped state $q_{i}$ in the previous layers.
% This suggests that the possible mechanism is the model performs single-step reasoning in continuous layers, resulting predicting the token $q_{i+1}$.
% This also interprets that why the index of layers with relative high probing results satisfies \( q_{i+2} > q_{i+1} > q_{i-1} \), as show in Figure~\ref{fig:multi_step_probe}.

\section{Noisy State Tracking}

We make an assumption that skipping, noise, and length generalization are present in real-world scenarios.
For example, consider the following example sampled from the OpenWebMath corpus.
The CoT reasoning misses a step in computing the derivative of $2x$ in the second to last line.
Nevertheless, given the vast scale of the pretraining corpus, it is hard to quantify the proportion of noisy reasoning.
Due to the existence of noise and the difficulty from vast corpora, we investigate the robustness of FSAs through controlled experiments.

\begin{quote}
    \textit{What is the derivative of} $f(x) = (e^{2x})(\ln(x))$?\\
    \textit{Mar 3, 2017}\\
    $f'(x) = e^{2x} \left(2 \ln x + \frac{1}{x}\right)$\\
    \textit{Explanation:}\\
    \textit{The derivative of} $\ln x$ \textit{is} $\frac{1}{x}$\\
    \textit{The derivative of} $e^{g(x)}$ \textit{is} $e^{g(x)} \cdot g'(x)$\\
    \textit{The derivative of} $h(x) \cdot l(x)$ \textit{is} $h'(x) \cdot l(x) + h(x) \cdot l'(x)$\\
    \textit{Then}\\
    $f'(x)$\\
    $= (e^{2x})' \cdot \ln x + e^{2x} \cdot \frac{1}{x}$\\
    $= e^{2x} \cdot 2 \cdot \ln x + e^{2x} \cdot \frac{1}{x}$\\
    $= e^{2x} \left(2 \ln x + \frac{1}{x}\right)$
\end{quote}

\section{Skipping Steps} \label{appendix skipping}

% There are two mechanisms for skipping in Transformer+CoT: one involves using multiple layers for state transitions, and the other uses a single layer for the same purpose. 
% The key difference is that the former relies on skipped states, while the latter does not. 
% To investigate the model's skipping behavior, we design experiments by suppressing skipped tokens(Refer to Appendix~\ref{appendix skipping} for more details).
% We find that suppressing a token leads to a significant drop in probing results, indicating that Transformer+CoT uses multiple layers for skipping. 
% This also explains the observed pattern where probing results follow a specific order in the layers.

According to the probing results in Figure~\ref{fig:multi_step_probe}, we propose two possible mechanisms for skipping in $\tt{Transformer}_{+CoT}$: one is to use continuous layers to achieve skip-step state transitions, and the other is to use a single layer to achieve skip-step transitions.
The key difference is that the former depends on the skipped states, while the latter does not.
To explore what algorithms the model uses in skipping, we design intervention experiments by suppressing skipped tokens.
Specifically, given the prompt $(m_1 \ldots m_n | q_1 \ldots q_{i-1})$, we evaluate the probing changes for $q_{i+1}$ at the last position, while suppressing token $q_{i}$.
We observe that suppressing $q_{i}$ significantly reduces positive probing results to nearly zero at the final position with skipping probabilities varying from 0 to 1, suggesting that $\tt{Transformer}_{+CoT}$ implements skipping through continuous layers.
% This interprets that why the index of layers with relative high probing results satisfies \( q_{i+2} > q_{i+1} > q_{i-1} \), as show in Figure~\ref{fig:multi_step_probe}.

% 774行
\section{Noise in Scratchpad} \label{appendix noise}

% According to the probing results in Section~\ref{subsec: noise}, we hypothesize that at $i\text{-th}$ step, the model suppresses the wrong input state $\hat{q_{i-1}}$ in the last layer, and attends to the intermediate representation at previous step, thus retrieving the correct input state $q_{i-1}$ and updating the state to $q_{i}$ successfully.
% We test this by masking the position ${q_{i-2}}$ in the last layer and evaluate the probing accuracy change for state $q_{i}$ in the last position.
% The results indicate that after masking, the MLP11 probing results dropped from greater than 0.9 to nearly zero.
% This suggests that attending to position ${q_{i-2}}$ in the last layer is necessary for correct state transition at current step.
% Moreover, we hold ablation experiments for model trained without noise, and find that probing accuracy change for state $q_{i}$ in the last position does not change much.
% This may suggest that the automata may develop adaptive attention patterns to deal with noise in scratchpad.

Given a corrupted prompt $(m_1 \ldots m_n | q_1 \ldots \hat{q}_{i-1})$ (where $\hat{q}_{i-1}$ represents an incorrect state), we probe the incorrect previous state $\hat{q}_{i-1}$, the correct previous state $q_{i-1}$, and the correct next state $q_i$ at the last position. 
The results in Figure~\ref{fig: noise} show that the activation maintains positive probing results for $\hat{q}_{i-1}$ across layers, up to the final one, where it exhibits a high average accuracy of 0.943 for $q_i$ and 0.379 for $q_{i-1}$.

In addition, we propose one possible mechanism that at the $i\text{-th}$ step, the model mitigates the influence of the incorrect input state $\hat{q}_{i-1}$ in the final layer through attending to the preceding position. This allows it to recover the correct input state $q_{i-1}$ and accurately update to $q_{i}$.
To verify this, we mask the position of $q_{i-2}$ in the final layer and analyze the resulting changes in probing accuracy for state $q_{i}$ at the last position. The results in Figure~\ref{fig: noise} show that after masking, the average score drops from above 0.9 to nearly zero, indicating that attending to $q_{i-2}$ in the last layer is crucial for accurate state transitions.
In contrast, we hold ablation experiments with a model trained without noise, and we observe minimal change in the probing accuracy for state $q_{i}$ while masking. 
This suggests that the FSA may develop adaptive attention mechanisms to deal with noise in the scratchpad.

\section{Length Generalization} \label{appendix length}

% There are many factors contributing to weak length generalization. 
% Here, we explain the model's failure in length generalization from the perspective of MLP neurons.
% We calculate activated neurons' precision and recall (Section~\ref{precision and recall}) on word problems with sequence length exceeding the training set.
% Results in Figure~\ref{fig:length_generalization} show that neurons' precision and recall decrease sharply, once the scratchpad step beyond the maximum step seen during training.
% Interestingly, we find that when the sequence length exceeds by a small margin, there is a ``U-turn'' phenomenon in precision during the last few steps of inference, which will yet disappear with the length increasing further.
% To explore the ``U-turn'' phenomenon, we conduct experiments with models train on various length ranges and find that the margin is approximately 150\%, beyond which the phenomenon disappears. 
% % Additionally, this phenomenon also exists when the test length is slightly shorter than the training length.
% This may indicate that activating correct neurons is not the prominent factor limiting length generalization.
% Finally, we conclude that improving the length generalization ability of the model requires other approches, perhaps the model structure, loss function or optimization algorithm.

Numerous factors contribute to poor length generalization in models. 
Here, we analyze the model's failure to generalize to longer sequences from the perspective of MLP neurons. 
We evaluate the precision and recall of activated neurons (see Section~\ref{precision and recall}) on word problems with sequence lengths exceeding those in the training set. 
As shown in Figure~\ref{fig:length_generalization}, both precision and recall drop significantly when the scratchpad steps surpass the maximum steps encountered during training. 
Notably, when the sequence length slightly exceeds the training range, a ``U-turn'' phenomenon emerges in precision during the final inference steps, which vanishes as the length increases further. 
To investigate this ``U-turn" phenomenon, we conducted experiments with models trained on various length ranges. 
Our findings reveal that this effect occurs within a margin of approximately 150\% of the training length, beyond which it disappears. 
The phenomenon suggests that activating the correct neurons may not be the primary factor limiting length generalization. 
In conclusion, enhancing the model's length generalization capability likely requires alternative approaches, such as modifications to the model architecture, loss function, or optimization algorithm.

\begin{figure}[ht]
  \includegraphics[width=\columnwidth]{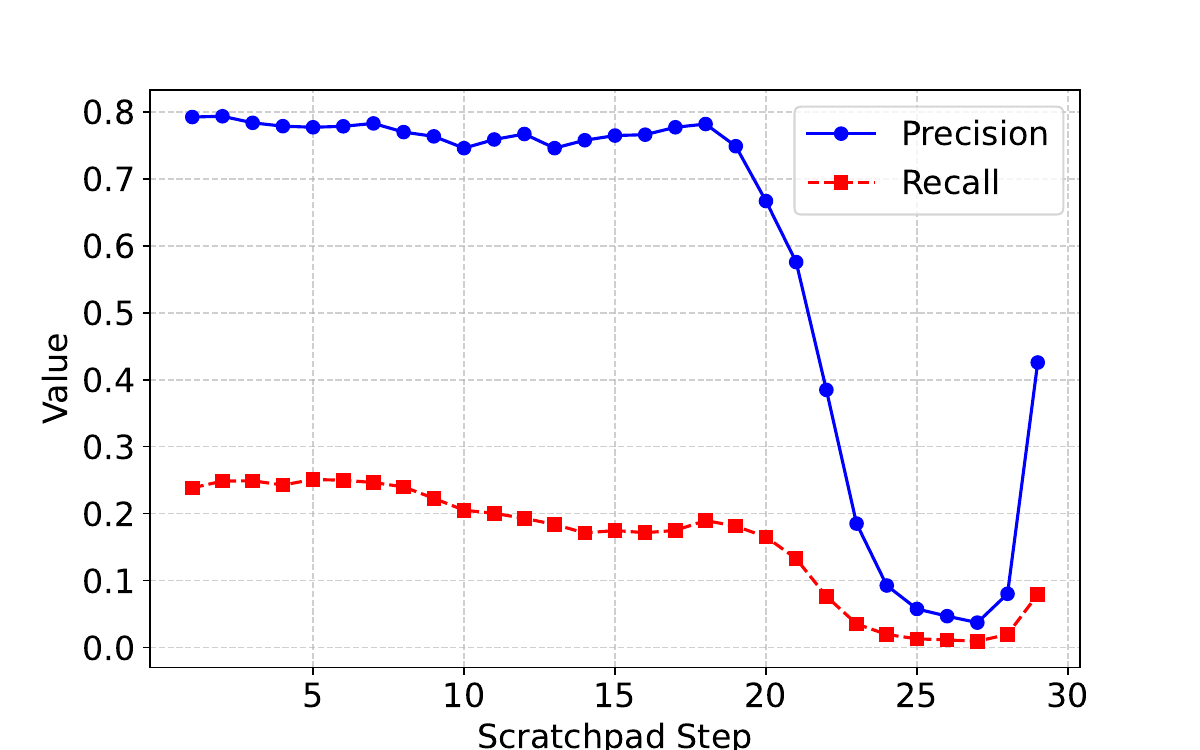}
  \caption{
  MLP11 neurons precision and recall across intermediate steps on word problems with length $30$.
  The model is trained on sequences with length ranging from 2 to 20.
  }
  \label{fig:length_generalization}
\end{figure}

\begin{figure*}[ht]
    \centering
    \begin{subfigure}{0.45\textwidth}
        \centering
        \includegraphics[width=\linewidth]{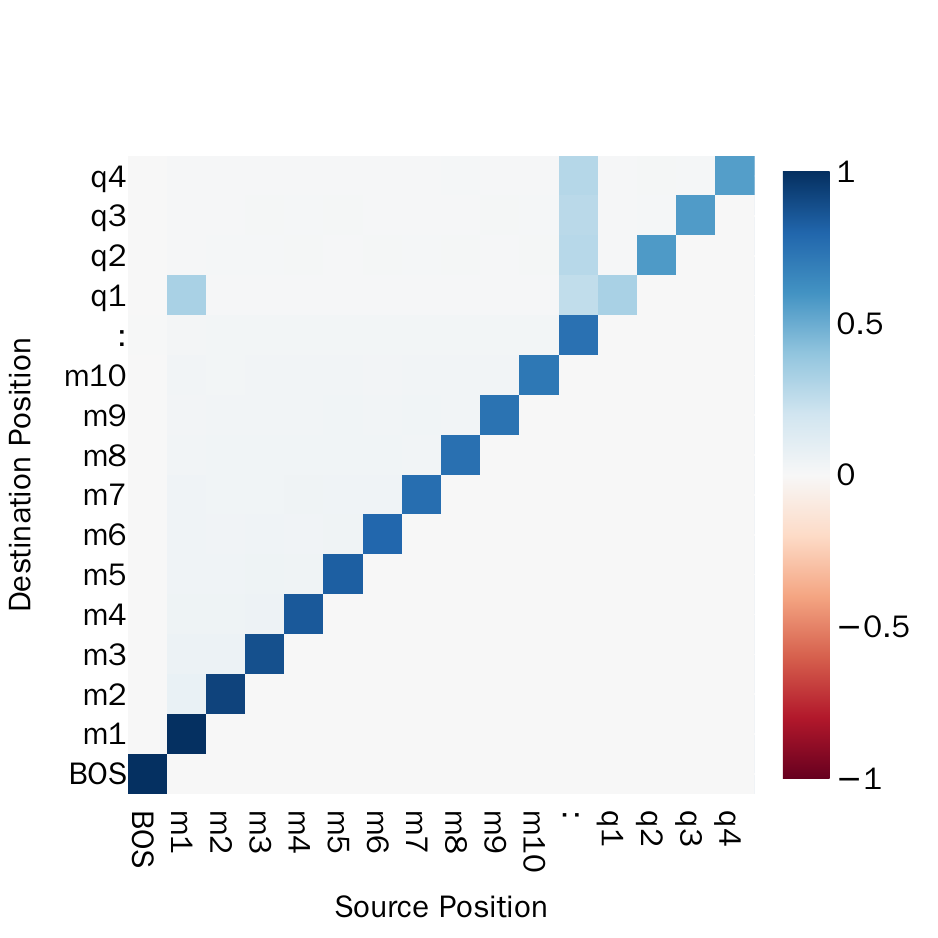}
        \caption{}
    \end{subfigure}
    \begin{subfigure}{0.45\textwidth}
        \centering
        \includegraphics[width=\linewidth]{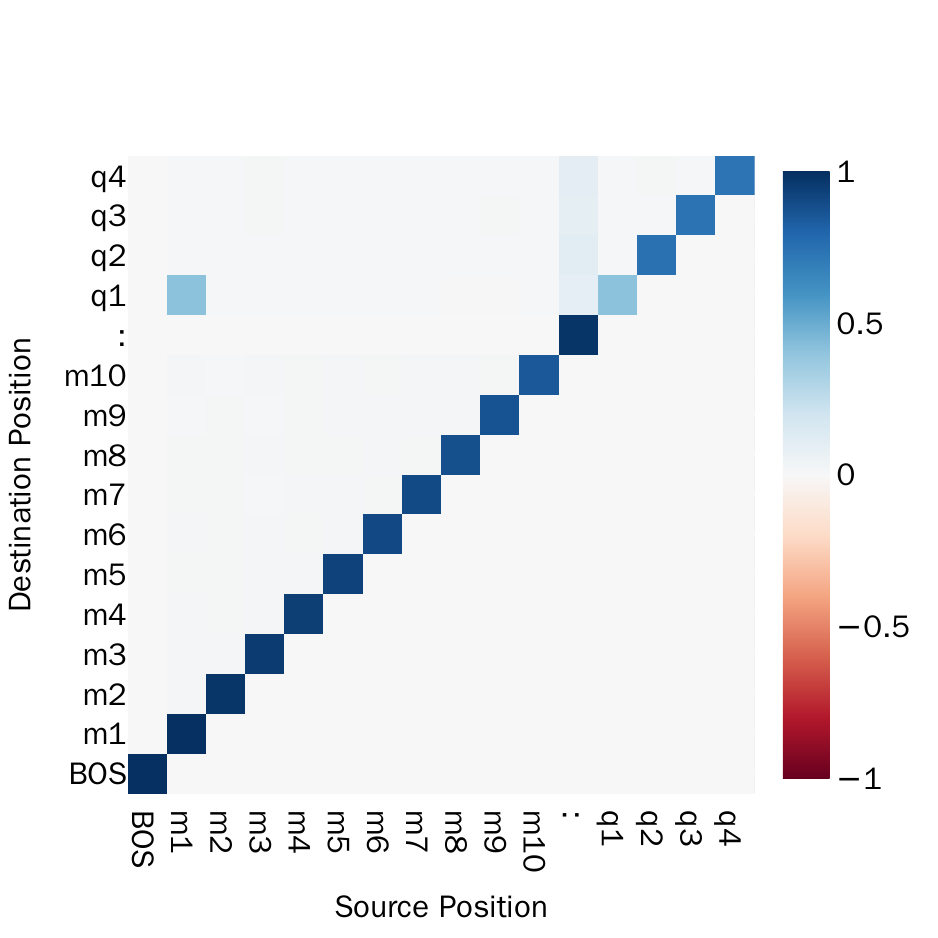}
        \caption{}
    \end{subfigure}
    \begin{subfigure}{0.45\textwidth}
        \centering
        \includegraphics[width=\linewidth]{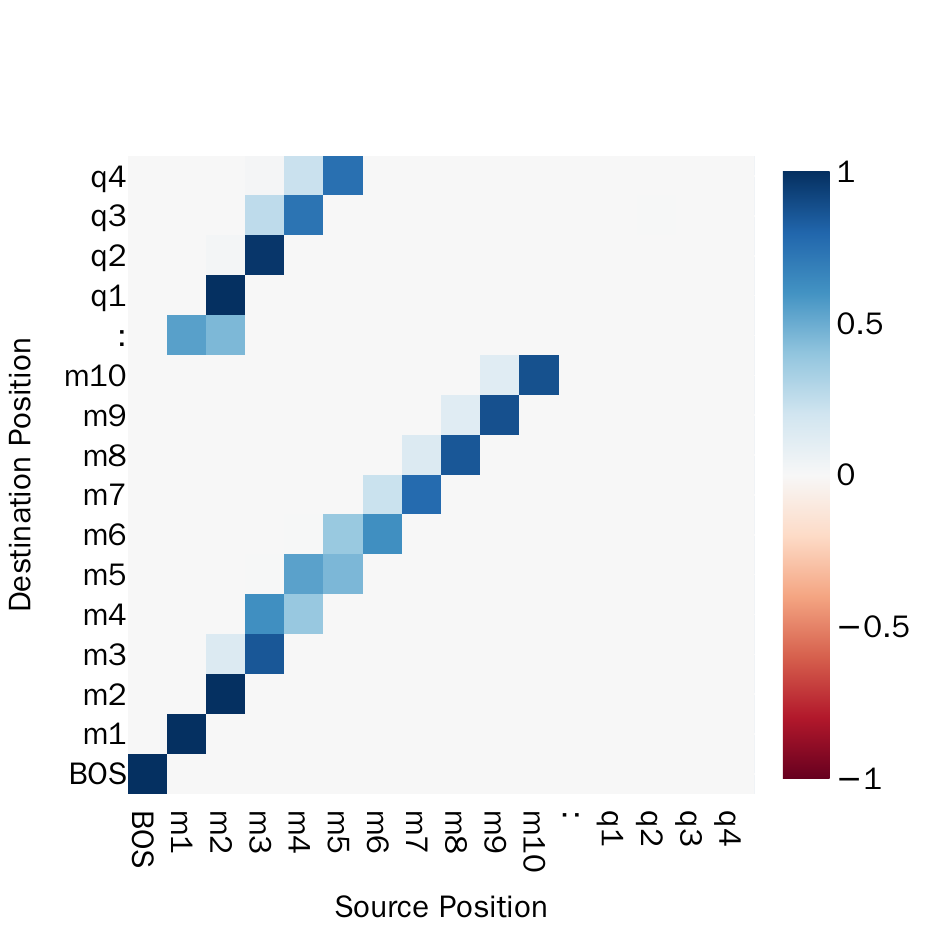}
        \caption{}
    \end{subfigure}
    \begin{subfigure}{0.45\textwidth}
        \centering
        \includegraphics[width=\linewidth]{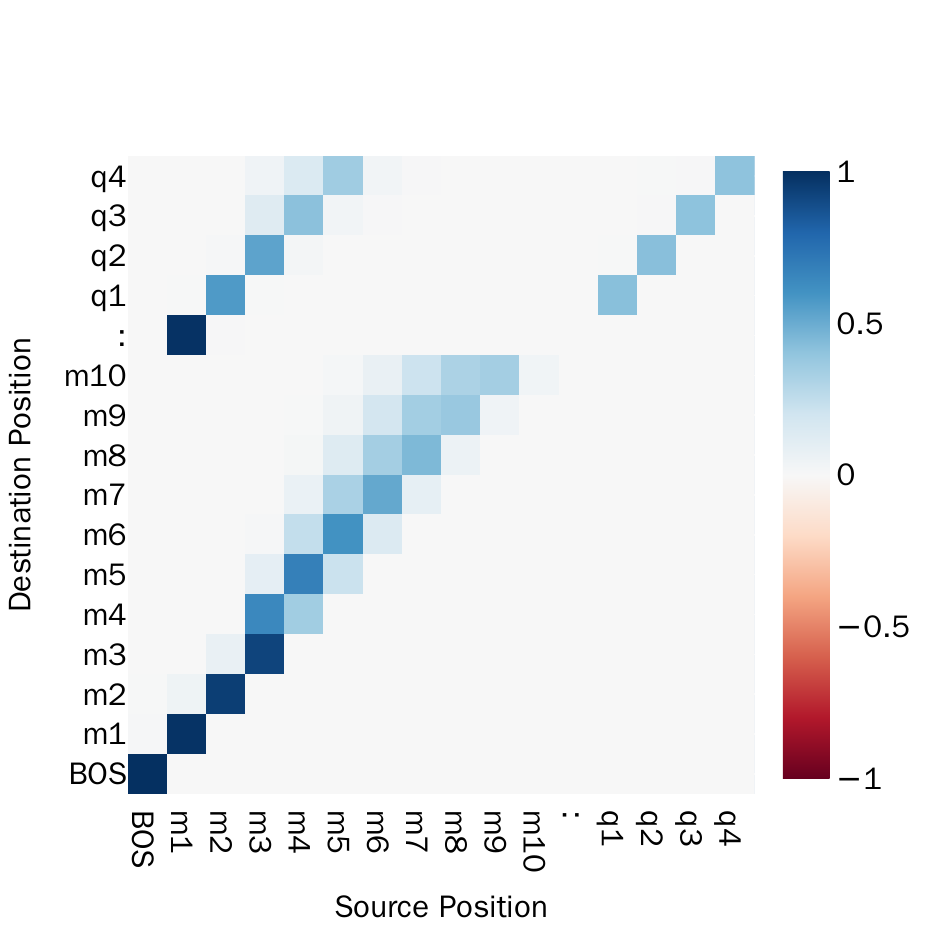}
        \caption{}
    \end{subfigure}
    \caption{Attention patterns of the principal attention heads.
    Specific attention heads send information directly from $m_i$ to $q_{i-1}$.
    Besides, the model learns another pattern: it passes information backward from $m_1$ to $m_n$ backwards until the relay position of ``:'', so that the model can attend to the relay position during inference.
    }
    \label{fig: attention}
\end{figure*}

\begin{figure*}[ht]
    \centering
    \begin{subfigure}{0.45\textwidth}
        \centering
        \includegraphics[width=\linewidth]{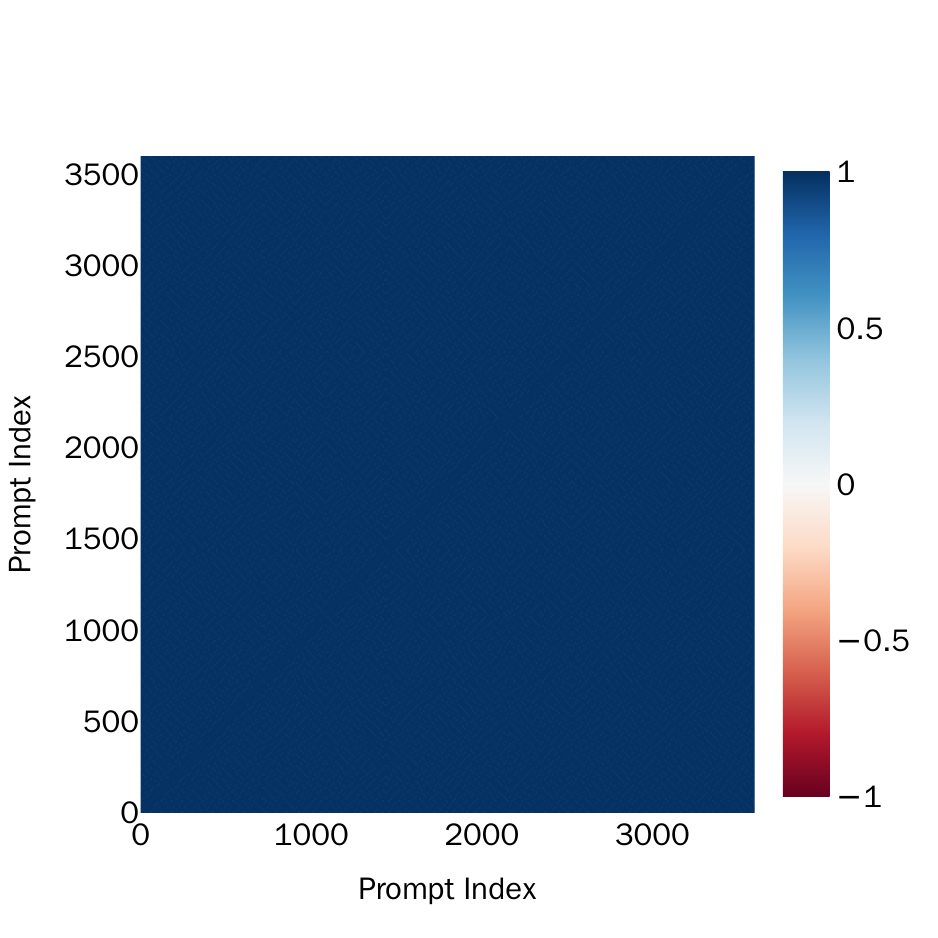}
        \caption{}
    \end{subfigure}
    \begin{subfigure}{0.45\textwidth}
        \centering
        \includegraphics[width=\linewidth]{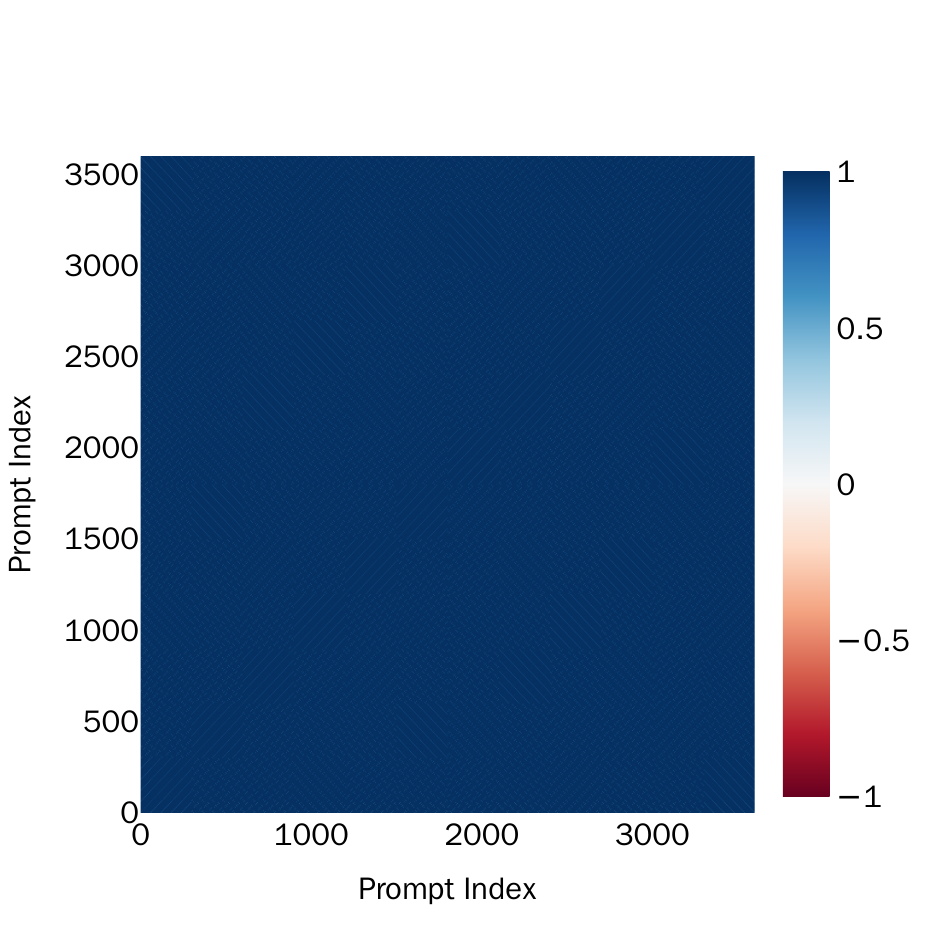}
        \caption{}
    \end{subfigure}
    \begin{subfigure}{0.45\textwidth}
        \centering
        \includegraphics[width=\linewidth]{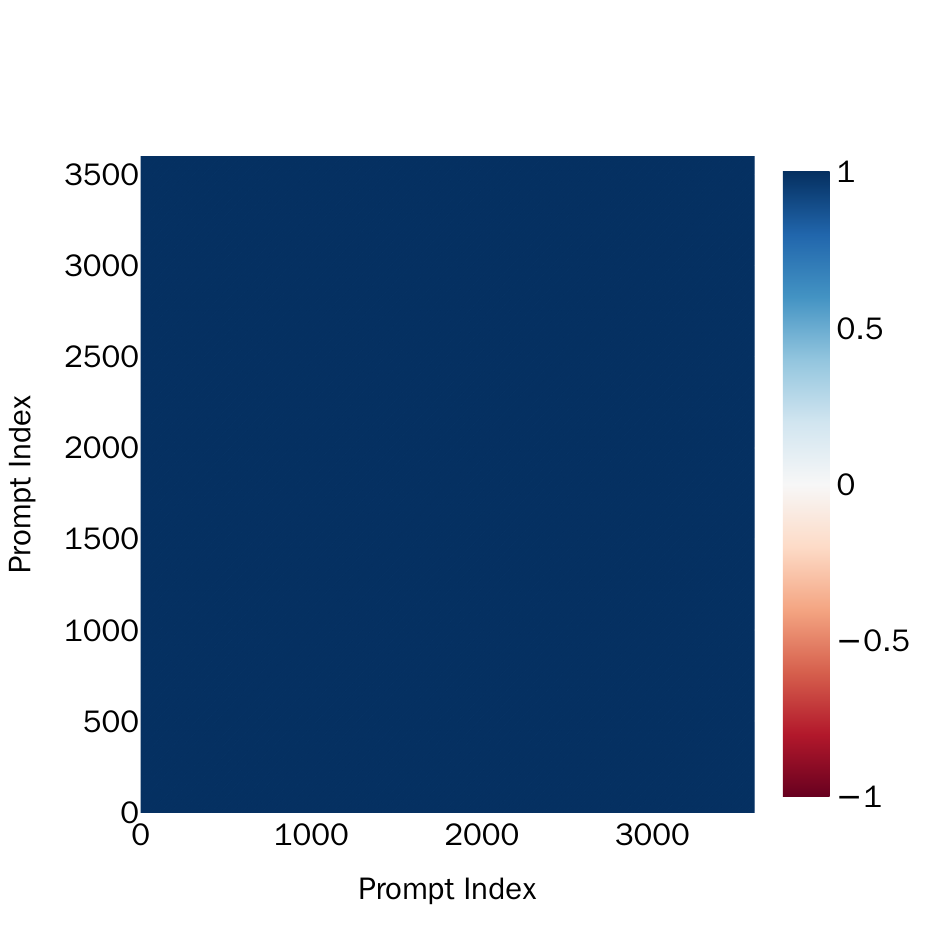}
        \caption{}
    \end{subfigure}
    \begin{subfigure}{0.45\textwidth}
        \centering
        \includegraphics[width=\linewidth]{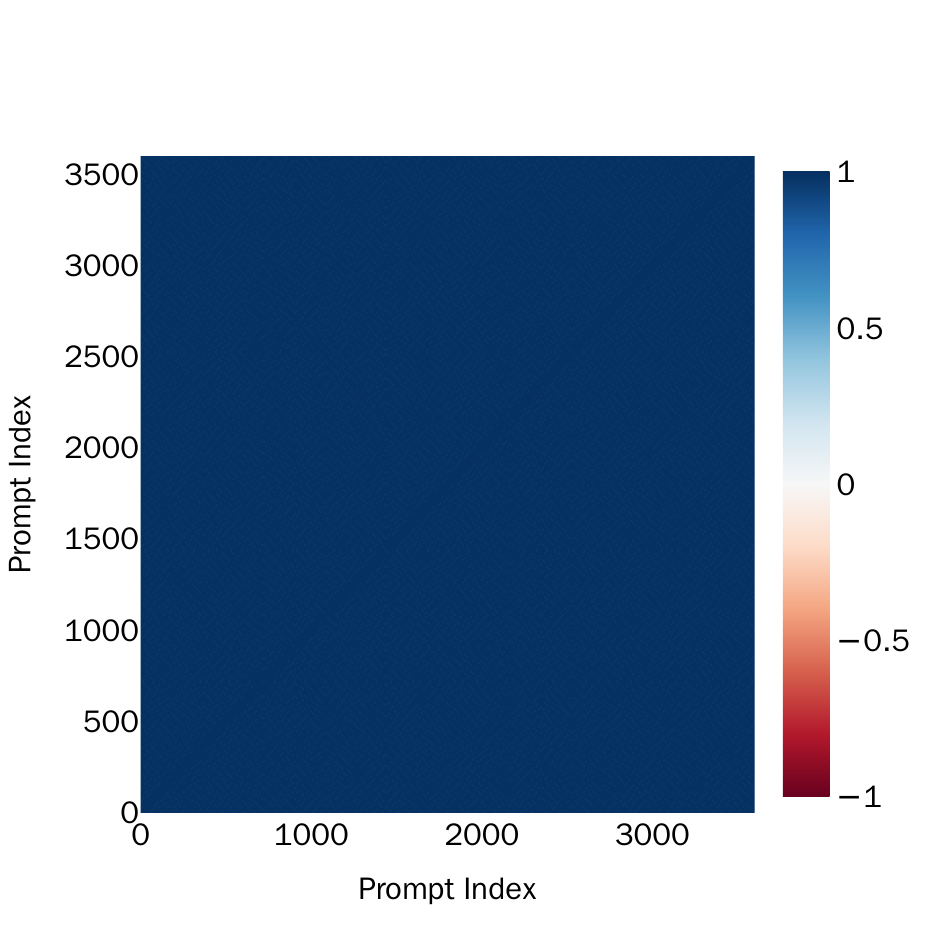}
        \caption{}
    \end{subfigure}
    \caption{(a): Compression and distinction metrics for MLP11 on $A_5$.
    (b): Compression and distinction metrics for MLP11 on $A_4 \times \mathbb{Z}_5$.
    (c): Compression and distinction metrics for MLP11 on $\mathbb{Z}_{60}$.
    (d): Compression and distinction metrics for MLP35 on $A_5$ with larger model size (GPT2-large).
    % The entire square's dark color represents that the metric of any pairwise combination of prompts is nearly 1.
    The entire square's dark color can be interpreted as the metric being nearly 1 for any pairwise combination of prompts.
    }
    \label{fig: stillexists}
\end{figure*}

\begin{figure*}[ht]
    \centering
    \begin{subfigure}{0.45\textwidth}
        \centering
        \includegraphics[width=\linewidth]{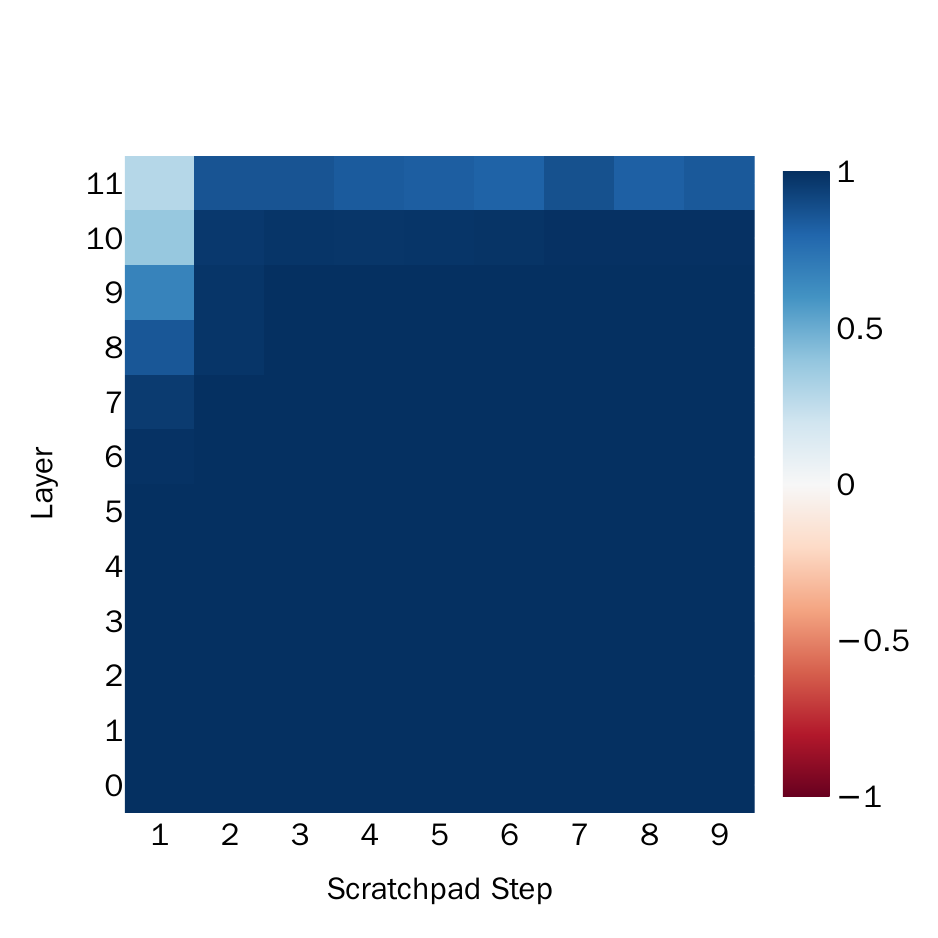}
        \caption{}
    \end{subfigure}
    \begin{subfigure}{0.45\textwidth}
        \centering
        \includegraphics[width=\linewidth]{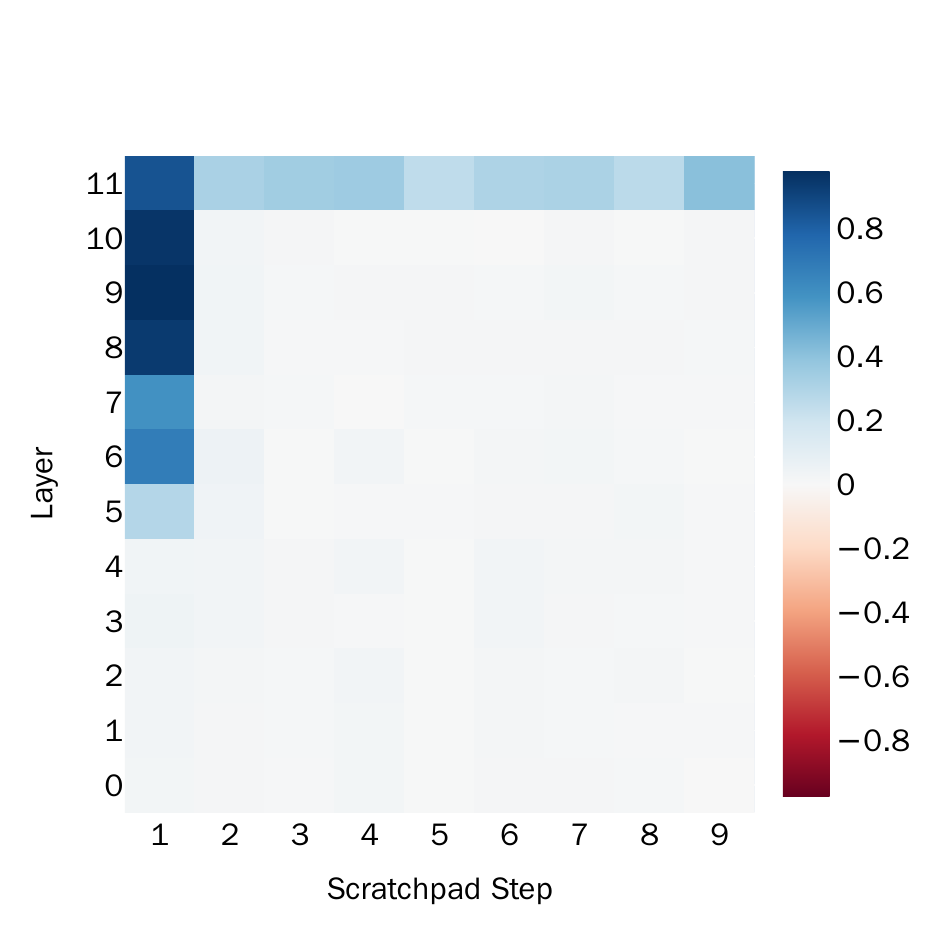}
        \caption{}
    \end{subfigure}
    \begin{subfigure}{0.45\textwidth}
        \centering
        \includegraphics[width=\linewidth]{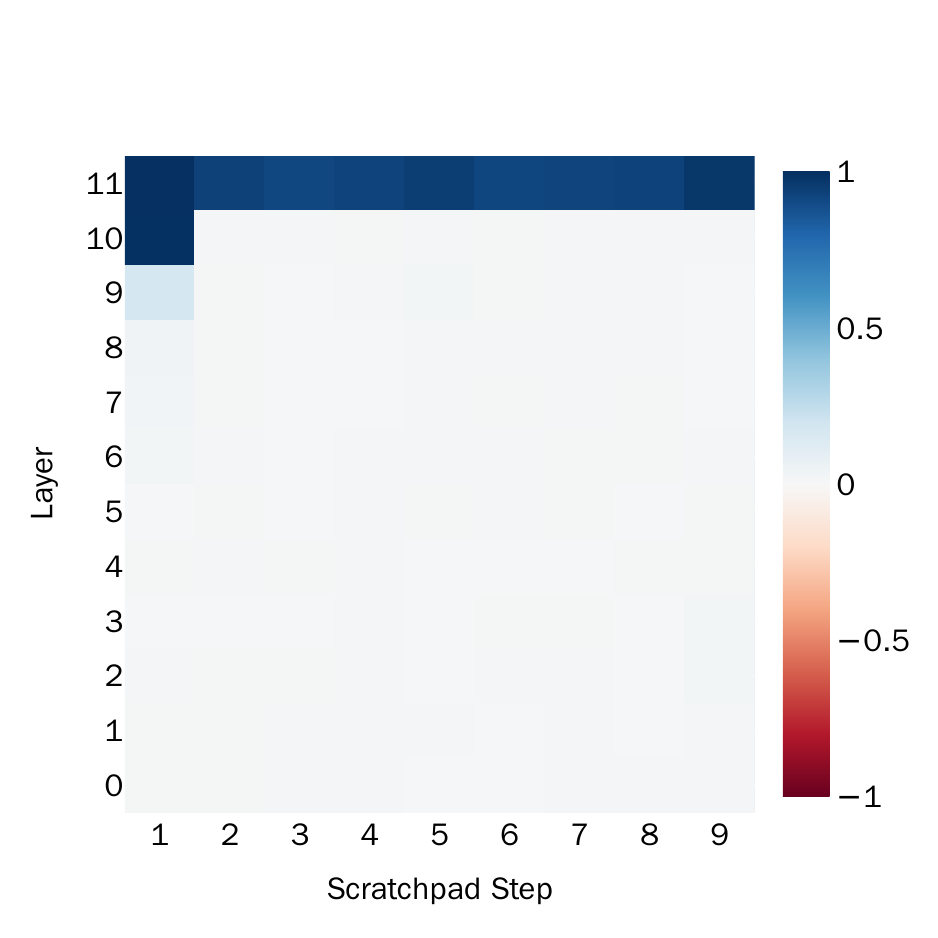}
        \caption{}
    \end{subfigure}
    \begin{subfigure}{0.45\textwidth}
        \centering
        \includegraphics[width=\linewidth]{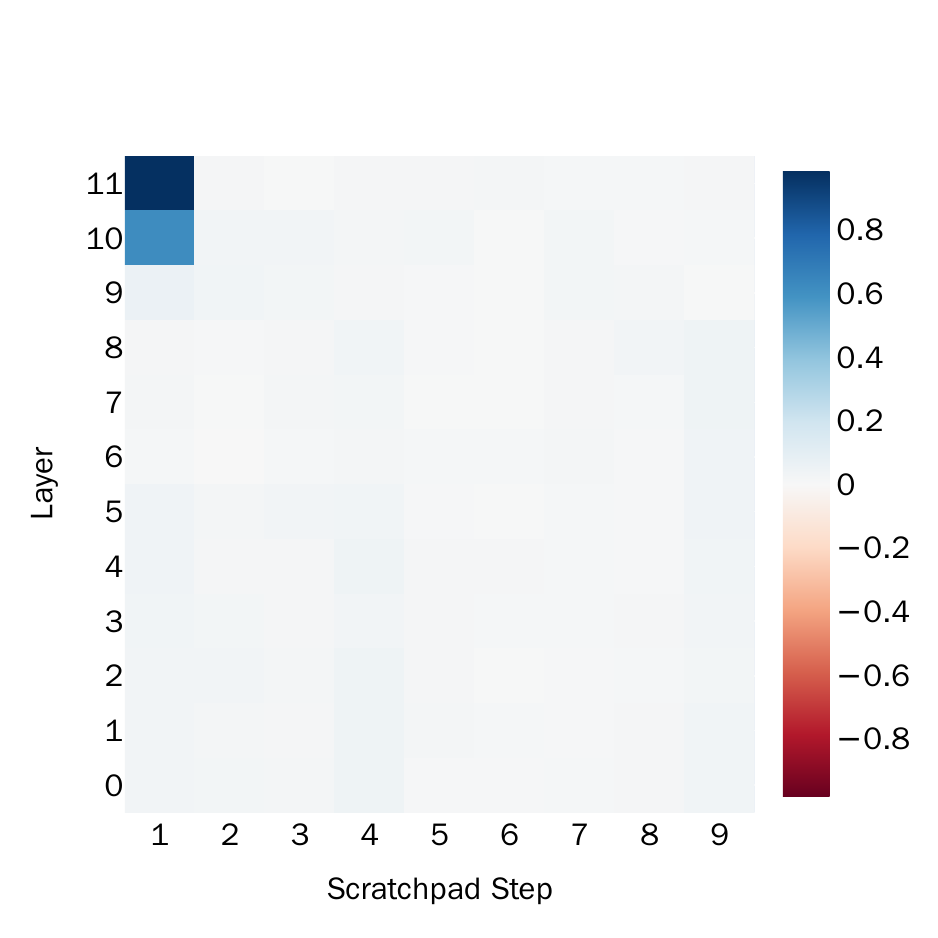}
        \caption{}
    \end{subfigure}
    \caption{(a): Probing results for $\hat{q}_{i-1}$.
    (b): Probing results for $q_{i-1}$.
    (c): Probing results for $q_i$.
    (d): Probing results for $q_i$ while masking the preceding position.
    % The model activation sustains high probing performance for $\hat{q}{i-1}$ across layers, culminating in the final layer with a high average accuracy of 0.943 for $q_i$ and 0.379 for $q{i-1}$. 
    % However, masking the position of $q_{i-2}$ causes the probing score to sharply decrease from over 0.9 to nearly zero.
    The model activation maintains high probing performance for $\hat{q}_{i-1}$ across layers, reaching a final layer average accuracy of 0.943 for $q_i$ and 0.379 for $q_{i-1}$.
    However, masking the position of $q_{i-2}$ leads to a sharp drop in the probing score—from above 0.9 to nearly zero.
    }
    \label{fig: noise}
\end{figure*}

% We conduct experiments with models train on various length ranges and find that the margin is roughly between 125\% and 150\%, beyond which the phenomenon disappears. 
% Additionally, this phenomenon also exists when the test length is slightly shorter than the training length.
% This may indicate that activating correct neurons is not the determining factor limiting length generalization.

\end{document}